\documentclass[final]{siamltex}
\usepackage{algorithm, algpseudocode}
\usepackage{amsfonts, amsmath, amsopn}
\usepackage{graphicx, subcaption, epstopdf}
\usepackage{hyperref}

\newenvironment{remark*}[2][Remark]{\par\bgroup{\scshape #1\ #2. }\it\ignorespaces}{\egroup}

\newcommand{\BALD}{\begin{aligned}}
\newcommand{\EALD}{\end{aligned}}
\newcommand{\BALDS}{\begin{aligned*}}
\newcommand{\EALDS}{\end{aligned*}}
\newcommand{\BCAS}{\begin{cases}}
\newcommand{\ECAS}{\end{cases}}
\newcommand{\BEAS}{\begin{eqnarray*}}
\newcommand{\EEAS}{\end{eqnarray*}}
\newcommand{\BEQ}{\begin{equation}}
\newcommand{\EEQ}{\end{equation}}
\newcommand{\BIT}{\begin{itemize}}
\newcommand{\EIT}{\end{itemize}}
\newcommand{\BMAT}{\begin{bmatrix}}
\newcommand{\EMAT}{\end{bmatrix}}
\newcommand{\BNUM}{\begin{enumerate}}
\newcommand{\ENUM}{\end{enumerate}}
\newcommand{\eg}{{\it e.g.}}

\newcommand{\ie}{{\it i.e.}}
\newcommand{\cf}{{\it cf.}}

\newcommand{\BA}{\begin{array}}
\newcommand{\EA}{\end{array}}

\newcommand{\reals}{\mathbf{R}}

\DeclareMathOperator*{\argmin}{\arg\min}

\DeclareMathOperator*{\minimize}{minimize}

\newcommand{\Prob}{\mathop{\mathbf{Pr}}}

\DeclareMathOperator{\sign}{sign}

\DeclareMathOperator{\trace}{tr}

\newcommand{\pc}{\hspace{1pc}}

\newcommand{\abs}[1]{\left| #1 \right|}

\newcommand{\norm}[1]{\left\| #1 \right\|}

\DeclareMathOperator{\logdet}{\log\det}

\DeclareMathOperator{\dom}{dom}

\DeclareMathOperator{\prox}{prox}

\title{Proximal Newton-type methods for minimizing composite functions}

\author{Jason D. Lee\footnotemark[1]\ \footnotemark[2]
  \and  Yuekai Sun\footnotemark[1]\ \footnotemark[2]
  \and  Michael A. Saunders\footnotemark[3]}

\begin{document}

\maketitle

\renewcommand{\thefootnote}{\fnsymbol{footnote}}

\footnotetext[1]{J. Lee and Y. Sun contributed equally to this work.}
\footnotetext[2]{Institute for Computational and Mathematical
  Engineering, Stanford University, Stanford, CA (\href{mailto:jdl17@stanford.edu}{\texttt{jdl17@stanford.edu}}, \href{mailto:yuekai@stanford.edu}{\texttt{yuekai@stanford.edu}}).}
\footnotetext[3]{Department of Management Science and Engineering,
  Stanford University, Stanford, CA (\href{mailto:saunders@stanford.edu}{\texttt{saunders@stanford.edu}}).}
\footnotetext[4]{Revised \today. A preliminary version of this work appeared in \cite{lee2012proximal}.}

\renewcommand{\thefootnote}{\arabic{footnote}}

\begin{abstract}
  We generalize Newton-type methods for minimizing smooth functions to
  handle a sum of two convex functions: a smooth function and a
  nonsmooth function with a simple proximal mapping. We show that the
  resulting proximal Newton-type methods inherit the desirable
  convergence behavior of Newton-type methods for minimizing smooth
  functions, even when search directions are computed inexactly. Many
  popular methods tailored to problems arising in bioinformatics,
  signal processing, and statistical learning are special cases of
  proximal Newton-type methods, and our analysis yields new
  convergence results for some of these methods.
\end{abstract}

%\begin{report}{Technical report SOL 2013-1}
%  Department of Management Science and Engineering, Stanford
%  University, Stanford, CA, \today.
%\end{report}



\pagestyle{myheadings}
\thispagestyle{plain}
\markboth{J. LEE, Y. SUN, AND M. SAUNDERS}{PROXIMAL NEWTON-TYPE METHODS}

\section{Introduction}
\label{sec:introduction}

Many problems of relevance in bioinformatics, signal processing, and
statistical learning can be formulated as minimizing a \emph{composite
  function}:
\begin{align}
  \minimize_{x \in \reals^n} \,f(x) := g (x) + h(x),
  \label{eq:composite-form}
\end{align}
where $g$ is a convex, continuously differentiable loss function, and
$h$ is a convex but not necessarily differentiable penalty function or
regularizer. Such problems include the \emph{lasso}
\cite{tibshirani1996regression}, the \emph{graphical lasso}
\cite{friedman2008sparse}, and trace-norm matrix completion
\cite{candes2009exact}.

We describe a family of Newton-type methods for minimizing composite
functions that achieve superlinear rates of convergence subject to
standard assumptions. The methods can be interpreted as
generalizations of the classic proximal gradient method that account
for the curvature of the function when selecting a search
direction. Many popular methods for minimizing composite functions are
special cases of these \emph{proximal Newton-type methods}, and our
analysis yields new convergence results for some of these methods.

In section \ref{sec:introduction} we review state-of-the-art methods
for problem \eqref{eq:composite-form} and related work on projected
Newton-type methods for constrained optimization.  In sections
\ref{sec:proximal-newton-type-methods} and
\ref{sec:convergence-results} we describe proximal Newton-type
methods and their convergence behavior, and in section
\ref{sec:experiments} we discuss some applications of these methods
and evaluate their performance.

\textbf{Notation:} The methods we consider are \emph{line search
  methods}, which produce a sequence of points $\{x_k\}$ according to
\[
  x_{k+1} = x_k + t_k\Delta x_k,
\]
where $t_k$ is a \emph{step length} and $\Delta x_k$ is a \emph{search
  direction}. When we focus on one iteration of an algorithm, we drop
the subscripts (\eg, $x_+ = x + t\Delta x$). All the methods we
consider compute search directions by minimizing local models of the
composite function $f$. We use an accent $\hat{\cdot}$ to denote these
local models (\eg, $\hat{f}_k$ is a local model of $f$ at the $k$-th
step).

\subsection{First-order methods}

The most popular methods for minimizing composite functions are
\emph{first-order methods} that use \emph{proximal mappings} to handle
the nonsmooth part $h$. SpaRSA \cite{wright2009sparse} is a popular
\emph{spectral projected gradient} method that uses a \emph{spectral
  step length} together with a \emph{nonmonotone line search} to
improve convergence. TRIP \cite{kim2010scalable} also uses a spectral
step length but selects search directions using a trust-region
strategy.

We can accelerate the convergence of first-order methods using ideas
due to Nesterov \cite{nesterov2003introductory}. This yields
\emph{accelerated first-order methods}, which achieve
$\epsilon$-suboptimality within $O(1/\sqrt{\epsilon})$ iterations
\cite{tseng2008accelerated}. The most popular method in this family is
the Fast Iterative Shrinkage-Thresholding Algorithm (FISTA)
\cite{beck2009fast}. These methods have been implemented in the
package TFOCS \cite{becker2011templates} and used to solve problems
that commonly arise in statistics, signal processing, and statistical
learning.

\subsection{Newton-type methods}

There are two classes of methods that generalize Newton-type methods
for minimizing smooth functions to handle composite functions
\eqref{eq:composite-form}.  \emph{Nonsmooth Newton-type methods}
\cite{yu2010quasi} successively minimize a local quadratic model of
the composite function $f$:
\[
  \hat{f}_k(y) = f(x_k) + \sup_{z\in\partial f(x_k)} z^T(y-x_k) +
  \frac{1}{2}(y-x_k)^TH_k(y-x_k),
\]
where $H_k$ accounts for the curvature of $f$.  (Although computing
this $\Delta x_k$ is generally not practical, we can exploit the
special structure of $f$ in many statistical learning problems.)
\emph{Proximal Newton-type methods} approximate only the smooth part
$g$ with a local quadratic model:
%to compute a search direction:
\[
  \hat{f}_k(y) = g(x_k) + \nabla g(x_k)^T(y-x_k) + \frac{1}{2}(y-x_k)^TH_k(y-x_k) + h(y),
\]
where $H_k$ is an approximation to $\nabla^2 g(x_k)$.  This idea
can be traced back to the \emph{generalized proximal point method} of
Fukushima and Min\'{e} \cite{fukushima1981generalized}. 

Proximal Newton-type methods are a special case of cost approximation
(Patriksson \cite{patriksson1998cost}).  In particular, Theorem 4.1
(convergence under Rule E) and Theorem 4.6 (linear convergence) of
\cite{patriksson1998cost} apply to proximal Newton-type
methods. Patriksson shows superlinear convergence of the exact
proximal Newton method, but does not analyze the quasi-Newton
approximation, nor consider the adaptive stopping criterion of Section
\ref{sec:convergence-inexact-proxnewton}. By focusing on Newton-type
methods, we obtain stronger and more practical results.

Many popular methods for minimizing composite functions are special
cases of proximal Newton-type methods. Methods tailored to a specific
problem include \texttt{glmnet} \cite{friedman2007pathwise}, \texttt{newglmnet}
\cite{yuan2012improved}, QUIC \cite{hsieh2011sparse}, and the
Newton-LASSO method \cite{olsen2012newton}. Generic methods include
\emph{projected Newton-type methods} \cite{schmidt2009optimizing,
  schmidt2011projected}, proximal quasi-Newton methods
\cite{schmidt2010graphical, becker2012quasi}, and the method of Tseng
and Yun \cite{tseng2009coordinate, lu2011augmented}.

This article is the full version of our preliminary work
\cite{lee2012proximal}, and section \ref{sec:convergence-results}
includes a convergence analysis of inexact proximal Newton-type
methods (\ie, when the subproblems are solve inexactly). Our main
convergence results are:
\BNUM
\item The proximal Newton and proximal quasi-Newton methods (with line search) converge superlinearly.
\item The inexact proximal Newton method (with unit step length) converges locally at a linear or superlinear rate depending on the forcing sequence.
\ENUM
We also
describe an adaptive stopping condition to decide how exactly (or
inexactly) to solve the subproblem, and we demonstrate the benefits
empirically.

There is a rich literature on generalized equations, such as monotone
inclusions and variational inequalities. Minimizing composite
functions is a special case of solving generalized equations, and
proximal Newton-type methods are special cases of Newton-type methods
for solving them \cite{patriksson1998cost}. We refer to Patriksson
\cite{patriksson1999nonlinear} for a unified treatment of descent
methods for solving a large class of generalized equations. 

\section{Proximal Newton-type methods}
\label{sec:proximal-newton-type-methods}

In problem \eqref{eq:composite-form} we assume $g$ and $h$ are closed,
convex functions, with $g$ continuously differentiable and its
gradient $\nabla g$ Lipschitz continuous. The function $h$ is not
necessarily differentiable everywhere, but its \emph{proximal mapping}
\eqref{eq:prox-mapping} can be evaluated efficiently. We refer to $g$
as ``the smooth part'' and $h$ as ``the nonsmooth part''. We assume
the optimal value is attained at some optimal solution $x^\star$, not
necessarily unique.

\subsection{The proximal gradient method}

The proximal mapping of a convex function $h$ at $x$ is 
\begin{equation}
  \prox_{h}(x) := \argmin_{y\in\reals^n}\,h(y) + \frac{1}{2}\norm{y-x}^2.
  \label{eq:prox-mapping}
\end{equation}
Proximal mappings can be interpreted as generalized projections
because if $h$ is the indicator function of a convex set,
$\prox_h(x)$ is the projection of $x$ onto the set. If $h$ is the
$\ell_1$ norm and $t$ is a step-length, then $\prox_{th}(x)$ is the
\emph{soft-threshold operation}:
\[
  \prox_{t\ell_1}(x) = \sign(x)\cdot\max\{\abs{x}-t,0\},
\]
where $\sign$ and $\max$ are entry-wise, and $\cdot$ denotes the
entry-wise product.

The \emph{proximal gradient method} uses the proximal mapping of the
nonsmooth part to minimize composite functions.  For some step length
$t_k$, the next iterate is 
$x_{k+1} = \prox_{t_kh}\left(x_k-t_k\nabla g(x_k)\right)$.
This is equivalent to
\begin{gather}
   x_{k+1} = x_k - t_kG_{t_kf}(x_k)
\\ G_{t_kf}(x_k) := \frac{1}{t_k}\left(x_k-\prox_{t_kh}(x_k-t_k\nabla g(x_k)) \right),
\label{eq:composite-gradient-step}
\end{gather}
where $G_{t_kf}(x_k)$ is a
\emph{composite gradient step}. Most first-order methods, including
SpaRSA and accelerated first-order methods, are variants of this
simple method. We note three properties of the composite gradient~step: 
\begin{enumerate}
\item Let $\hat{g}$ be a simple quadratic model of $g$ near $x_k$
  (with $H_k$ a multiple of $I$):
  \[
    \hat{g}_k(y) := g(x_k) + \nabla g(x_k)^T(y-x_k) + \frac{1}{2t_k}
    \norm{y - x_k}^2 .
  \]
  The composite gradient step moves to the minimum of $\hat{g}_k + h$:
\begin{align}
   x_{k+1} &= \prox_{t_kh}\left(x_k-t_k\nabla g(x_k)\right)
\\        &= \argmin_y\,t_kh(y) + \frac{1}{2}\norm{y-x_k+t_k\nabla g(x_k)}^2
\\        &= \argmin_y\,\nabla g(x_k)^T(y-x_k) + \frac{1}{2t_k}\norm{y-x_k}^2 + h(y).
			 \label{eq:simple-quadratic}
\end{align}

\item The composite gradient step is neither a gradient nor a subgradient of $f$
  at any point; rather it is the sum of an explicit gradient (at $x$) and an
  implicit subgradient (at $\prox_h(x)$). The first-order optimality conditions of    	 
  \eqref{eq:simple-quadratic} are
  \[
    \partial h(x_{k+1}) + \frac{1}{t_k}(x_{k+1} - x_k) = 0.
  \]
  We express $x_{k+1} - x_k$ in terms of $G_{t_kf}$ and rearrange to obtain
  \[
    G_{t_kf}(x_k) \in \nabla g(x_k) + \partial h(x_{k+1}).
  \]

\item The composite gradient step is zero if and only if $x$ minimizes
  $f$, \ie{} $G_f(x) = 0$ generalizes the usual (zero gradient)
  optimality condition to composite functions (where $G_f(x) =
  G_{tf}(x)$ when $t=1$).
\end{enumerate}
%The third property generalizes the zero gradient optimality condition
%for smooth functions to composite functions.
We shall use the length of $G_f(x)$ to measure the optimality of a
point $x$.
We show that $G_f$ inherits the Lipschitz continuity of $\nabla g$.

\begin{definition}
\label{asu:lipschitz-continuous}
A function $F$ is Lipschitz continuous with constant $L$ if
\begin{equation}
\norm{F(x) - F(y)} \le L\norm{x - y}\text{ for any }x,y.
\label{eq:lipschitz-continuity}
\end{equation}
\end{definition}

\begin{lemma}
  \label{lem:G-lipschitz}
  If $\nabla g$ is Lipschitz continuous with constant $L_1$, then
  \[
	\norm{G_f(x)} \le (L_1 + 1)\norm{x - x^\star}.
  \]
\end{lemma}

\begin{proof}
  The composite gradient steps at $x$ and the optimal solution $x^\star$ satisfy
  \begin{align*}
	G_f(x) &\in \nabla g(x) + \partial h(x - G_f(x)), \\
	G_f(x^\star) &\in \nabla g(x^\star) + \partial h(x^\star).
  \end{align*}
  We subtract these two expressions and rearrange to obtain
  \[
    \partial h(x - G_f(x)) - \partial h(x^\star) \ni G_f(x) - (\nabla g(x) - \nabla g(x^\star)).
  \]
  Since $h$ is convex, $\partial h$ is monotone and 
  \begin{align*}
    0 &\le (x - G_f(x) - x^\star)^T\partial h(x - G_f(x_k))
  \\  &= -G_f(x)^TG_f(x) + (x - x^\star)^TG_f(x) + G_f(x)^T(\nabla g(x) - \nabla g(x^\star))
  \\  &\pc + (x - x^\star)^T(\nabla g(x) - \nabla g(x^\star)).
  \end{align*}
  We drop the last term because it is nonnegative ($\nabla g$ is monotone) to obtain
  \begin{align*}
     0 &\le -\norm{G_f(x)}^2 + (x - x^\star)^TG_f(x) + G_f(x)^T(\nabla g(x) - \nabla g(x^\star))
  \\   &\le -\norm{G_f(x)}^2 + \norm{G_f(x)}(\norm{x - x^\star} + \norm{\nabla g(x) - \nabla g(x^\star)}),
  \end{align*}
  so that
  \begin{equation}
	\norm{G_f(x)} \le \norm{x - x^\star} + \norm{\nabla g(x) - \nabla g(x^\star)}.
	\label{eq:G-lipschitz-1}
  \end{equation}
  Since $\nabla g$ is Lipschitz continuous, we have
  \[
	\norm{G_f(x)} \le (L_1 + 1)\norm{x - x^\star}.
  \]
\end{proof}

\subsection{Proximal Newton-type methods}

Proximal Newton-type methods use a symmetric positive definite
matrix $H_k \approx \nabla^2 g(x_k)$ to model the curvature of $g$:
\[
  \hat{g}_k(y) = g(x_k) + \nabla g(x_k)^T(y-x_k) + \frac{1}{2}(y-x_k)^TH_k(y-x_k).
\]
A proximal Newton-type search direction $\Delta x_k$ solves the subproblem
\begin{equation}
   \Delta x_k = \argmin_d\,\hat{f}_k(x_k + d) := \hat{g}_k(x_k + d) + h(x_k+d).
%  & = \argmin_d\, Q_k(x_k + d) + h(x_k+d)
  \label{eq:proxnewton-search-dir-1}
\end{equation}
There are many strategies for choosing $H_k$. If we choose $H_k =
\nabla^2 g(x_k)$, we obtain the \emph{proximal Newton method}. If
we build an approximation to $\nabla^2 g(x_k)$ according to a
quasi-Newton strategy, we obtain a \emph{proximal quasi-Newton
  method}. If the problem is large, we can use limited memory
quasi-Newton updates to reduce memory usage. Generally speaking, most
strategies for choosing Hessian approximations in Newton-type methods
(for minimizing smooth functions) can be adapted to choosing $H_k$ in
proximal Newton-type methods.

When $H_k$ is not positive definite, we can also adapt strategies for
handling indefinite Hessian approximations in Newton-type methods. The
most simple strategy is Hessian modification: we add a multiple of the
identity to $H_k$ when $H_k$ is indefinite. This makes the subproblem
strongly convex and damps the search direction. In a proximal
quasi-Newton method, we can also do update skipping: if an update
causes $H_k$ to become indefinite, simply skip the update.

We can also express the proximal Newton-type search direction using
\emph{scaled proximal mappings}. This lets us interpret a proximal
Newton-type search direction as a ``composite Newton step'' and
reveals a connection with the composite gradient step.

\begin{definition}
  \label{def:scaled-prox}
  Let $h$ be a convex function and $H$ be a positive definite matrix. Then
  the scaled proximal mapping of $h$ at $x$ is
  \begin{align}
	\prox_h^H(x) := \argmin_{y\in\reals^n}\,h(y)+\frac{1}{2}\norm{y-x}_H^2.
	\label{eq:scaled-prox}
  \end{align}
\end{definition}

Scaled proximal mappings share many properties with (unscaled)
proximal mappings:
\begin{enumerate}
\item The scaled proximal point $\prox_h^H(x)$ exists and is unique for $x\in\dom h$ because the
  proximity function is strongly convex if $H$ is positive definite.

\item Let $\partial h(x)$ be the subdifferential of $h$ at $x$. Then
  $\prox_h^H(x)$ satisfies
  \begin{align}
	H\left(x-\prox_h^H(x)\right)\in\partial h\left(\prox_h^H(x)\right).
	\label{eq:scaled-prox-1}
  \end{align}

\item Scaled proximal mappings are \emph{firmly nonexpansive} in the
  $H$-norm. That is, if $u = \prox_h^H(x)$ and $v = \prox_h^H(y)$,
  then
  \[
	 (u-v)^TH(x-y)\ge \norm{u-v}_H^2,
  \]
  and the Cauchy-Schwarz inequality implies
  $\norm{u-v}_H \le \norm{x-y}_H$.
\end{enumerate}

We can express proximal Newton-type search directions as
``composite Newton steps'' using scaled proximal mappings:
\begin{align}
  \Delta x = \prox_h^{H}\left(x-H^{-1}\nabla g(x)\right) - x.
  \label{eq:proxnewton-search-direction-scaled-prox}
\end{align}
%We verify this is equivalent to \eqref{eq:proxnewton-search-dir-1}:
%\begin{align*}
%x_k + \Delta x_k &= \prox_h^{H_k}\left(x_k-H_k^{-1}\nabla g(x_k)\right) \\
%&= \argmin_y\,h(y) + \frac{1}{2}\|d+H_k^{-1}\nabla g(x_k)\|_{H_k}^2 \\
%&= \argmin_y\,\nabla g(x_k)^Td + \frac{1}{2}d^TH_kd + h(x_k + d).
%\end{align*}
We use \eqref{eq:scaled-prox-1} to deduce that proximal Newton
search directions satisfy
\[
  H\left(H^{-1}\nabla g(x) - \Delta x\right) \in \partial h(x + \Delta x).
\]
We simplify to obtain
\begin{align}
  H\Delta x \in -\nabla g(x) - \partial h(x + \Delta x).
  \label{eq:proxnewton-search-direction-2}
\end{align}
Thus proximal Newton-type search directions, like composite gradient
steps, combine an explicit gradient with an implicit subgradient. This
expression reduces to the Newton system when $h=0$.

\begin{proposition}[Search direction properties]
  \label{prop:search-direction-properties}
  If $H$ is positive definite, then $\Delta x$ in \eqref{eq:proxnewton-search-dir-1}
  satisfies
  \begin{gather}
	f(x_+) \le f(x)+t\left(\nabla g(x)^T\Delta x+h(x+\Delta x)-h(x)\right)+ O(t^2),
	\label{eq:search-direction-properties-1}
	\\ \nabla g(x)^T\Delta x+h(x+\Delta x)-h(x) \le -\Delta x^TH\Delta x. 
	\label{eq:search-direction-properties-2}
  \end{gather}
\end{proposition}

\begin{proof}
  For $t\in(0,1]$, 
  \begin{align*}
	f(x_+) -f(x) &= g(x_+)-g(x)+h(x_+)-h(x)
  \\ &\le g(x_+)-g(x)+th(x+\Delta x)+(1-t)h(x)-h(x)
  \\ &= g(x_+)-g(x)+t(h(x+\Delta x)-h(x))
  \\ &= \nabla g(x)^T(t\Delta x)+t(h(x+\Delta x)-h(x))+O(t^2),
  \end{align*}
  which proves \eqref{eq:search-direction-properties-1}.

  Since $\Delta x$ steps to the minimizer of $\hat{f}$
  \eqref{eq:proxnewton-search-dir-1}, $t\Delta x$ satisfies
  \begin{align*}
	 &\nabla g(x)^T\Delta x+\frac{1}{2}\Delta x^TH\Delta x+h(x+\Delta x)
  \\ &\pc\le \nabla g(x)^T(t\Delta x)+\frac12 t^2 \Delta x^TH\Delta x+h(x_+)
  \\ &\pc\le t\nabla g(x)^T\Delta x+\frac12 t^2 \Delta x^TH\Delta x+th(x+\Delta x)+(1-t)h(x).
  \end{align*}
  We rearrange and then simplify:
  \begin{gather*}
	 (1-t)\nabla g(x)^T\Delta x + \frac{1}{2}(1-t^2)\Delta x^TH\Delta x
								+ (1-t)(h(x+\Delta x)-h(x)) \le 0
  \\ \nabla g(x)^T\Delta x+\frac{1}{2}(1+t)\Delta x^TH\Delta x+h(x+\Delta x)-h(x) \le 0
  \\ \nabla g(x)^T\Delta x+h(x+\Delta x)-h(x) \le -\frac{1}{2}(1+t)\Delta x^TH\Delta x.
  \end{gather*} 
  Finally, we let $t\to 1$ and rearrange to obtain
  \eqref{eq:search-direction-properties-2}.
\end{proof}

Proposition \ref{prop:search-direction-properties} implies the search
direction is a descent direction for $f$ because we can substitute
\eqref{eq:search-direction-properties-2} into
\eqref{eq:search-direction-properties-1} to obtain
\begin{align}
  f(x_+) \le f(x)-t\Delta x^TH\Delta x+O(t^2).
  \label{eq:descent}
\end{align}

\begin{proposition}
  \label{prop:first-order-conditions}
  Suppose $H$ is positive definite. Then $x^\star$ is an optimal
  solution if and only if at $x^\star$ the search direction $\Delta x$
  \eqref{eq:proxnewton-search-dir-1} is zero.
%   \[
%     0 = \argmin_d\,\hat{f}(x^\star+d).
%   \]
\end{proposition}

\begin{proof}
  If $\Delta x$ at $x^\star$ is nonzero, then it is a descent direction for
  $f$ at $x^\star$. Clearly $x^\star$ cannot be a minimizer of $f$.
  
  If $\Delta x = 0$, then $x$ is the minimizer of $\hat{f}$, so that
  \[
	\nabla g(x)^T(td)+\frac12 t^2 d^THd+h(x+td) -h(x)\ge 0
  \]
  for all $t>0$ and $d$. We rearrange to obtain
  \begin{align}
	h(x+td)-h(x) \ge -t\nabla g(x)^Td-\frac12 t^2 d^THd. \label{eq:first-order-cond-1}
  \end{align}
  Let $Df(x,d)$ be the directional derivative of $f$ at $x$ in the direction $d$:
  \begin{align}
	Df(x,d) &= \lim_{t\to 0} \frac{f(x+td)-f(x)}{t} 
  \nonumber \\
  &= \lim_{t\to 0} \frac{g(x+td)-g(x)+h(x+td)-h(x)}{t} 
  \nonumber \\
  &= \lim_{t\to 0} \frac{t\nabla g(x)^Td+O(t^2)+h(x+td)-h(x)}{t}.
  \label{eq:first-order-cond-2}
  \end{align}
  We substitute \eqref{eq:first-order-cond-1} into
  \eqref{eq:first-order-cond-2} to obtain
  \begin{align*}
	 Df(x,u) &\ge \lim_{t\to 0}
				  \frac{t\nabla g(x)^Td+O(t^2) - \frac12 t^2 d^THd-t\nabla g(x)^Td}{t}
  \\         &=   \lim_{t\to 0} \frac{-\frac12 t^2 d^THd+O(t^2)}{t} = 0.
  \end{align*}
  Since $f$ is convex, $x$ is an optimal solution if and only if $\Delta x = 0$. 
\end{proof}

In a few special cases we can derive a closed-form expression for the
proximal Newton search direction, but usually we must resort to an
iterative method to solve the subproblem \eqref{eq:proxnewton-search-dir-1}.
The user should choose an iterative method that exploits the
properties of $h$.  For example, if $h$ is the $\ell_1$ norm, (block)
coordinate descent methods combined with an active set strategy are
known to be very efficient \cite{friedman2007pathwise}.

We use a line search procedure to select a step length $t$ that
satisfies a sufficient descent condition: the next iterate $x_+$ satisfies $f(x_+) \le f(x)+\alpha t\lambda$, where
\BEQ
   \lambda := \nabla g(x)^T\Delta x+h(x+\Delta x)-h(x).
   \label{eq:sufficient-descent}
\EEQ
The parameter $\alpha\in(0,0.5)$ can be interpreted as the fraction of the
decrease in $f$ predicted by linear extrapolation that we will accept.
A simple example is a
\emph{backtracking line search} \cite{boyd2004convex}: backtrack along
the search direction until a suitable step length is
selected. Although simple, this procedure performs admirably in
practice.

An alternative strategy is to search along the \emph{proximal arc},
\ie, the arc/curve
\begin{equation}
\Delta x_k(t) := \argmin_y\,\nabla g(x_k)^T(y-x_k) + \frac{1}{2t}(y-x_k)^TH_k(y-x_k) + h(y).
\label{eq:arc-search-subproblem}
\end{equation}
Arc search procedures have some benefits relative to line search
procedures. First, the arc search step is the optimal solution to a subproblem. Second, when the optimal solution lies on a low-dimensional
manifold of $\reals^n$, an arc search strategy is likely to identify
this manifold. The main drawback is the cost of obtaining trial
points: a subproblem must be solved at each trial point.

\begin{lemma}
  \label{lem:acceptable-step-lengths}
  Suppose $H\succeq mI$ for some $m > 0$ and $\nabla g$ is Lipschitz
  continuous with constant $L_1$. Then the sufficient descent
  condition \eqref{eq:sufficient-descent} is satisfied by
  \begin{align}
	t \le \min \left\{1,\frac{2m}{L_1}(1-\alpha)\right\}.
	\label{eq:step-length-conditions}
  \end{align}
\end{lemma}

\begin{proof}
  We can bound the decrease at each iteration by
  \begin{align*}
	 & f(x_+) -f(x) = g(x_+) -g(x) +h(x_+) -h(x)
  \\ &\pc\le \int_0^1 \nabla g(x+s(t\Delta x))^T(t\Delta x) ds
		+ th(x+\Delta x) +(1-t)h(x) -h(x)
  \\ &\pc=\nabla g(x)^T(t\Delta x)+ t(h(x+\Delta x) -h(x))
       +\int_0^1 (\nabla g(x+s(t\Delta x))-\nabla g(x))^T(t\Delta x) ds
  \\ &\pc\le t\left(\nabla g(x)^T\Delta x+h(x+\Delta x) -h(x)
      + \int_0^1 \norm{\nabla g(x+s(\Delta x))-\nabla g(x)}\norm{\Delta x} ds\right).
\end{align*}  
Since $\nabla g$ is Lipschitz continuous with constant $L_1$,
\begin{align}  
  f(x_+) -f(x) &\le t\left(\nabla g(x)^T\Delta x+h(x+\Delta x)
						   - h(x)+\frac{L_1t}{2} \norm{\Delta x}^2\right) 
   \nonumber
\\             &= t\left(\lambda+\frac{L_1t}{2} \norm{\Delta x}^2\right).
   \label{eq:acceptable-step-lengths-1}
\end{align}
If we choose
$t\le \frac{2m}{L_1}(1-\alpha)$, then
\[
   \frac{L_1t}{2} \norm{\Delta x}^2 \le m(1-\alpha)\norm{\Delta x}^2
   \le (1-\alpha)\Delta x^T H\Delta x.
\]
By \eqref{eq:search-direction-properties-2}, we have
$\frac{L_1t}{2} \norm{\Delta x}^2 \le -(1-\alpha)\lambda$. We substitute this expression into \eqref{eq:acceptable-step-lengths-1} to obtain the desired expression:
$$
f(x_+) -f(x) \le t\left(\lambda-(1-\alpha)\lambda\right) = t(\alpha\lambda).
$$
\end{proof}

\begin{algorithm}
\caption{A generic proximal Newton-type method} 
\label{alg:prox-newton} 
\begin{algorithmic}[1]
\Require starting point $x_0\in\dom f$
\Repeat
\State Choose $H_k$, a positive definite approximation to the Hessian.
\State Solve the subproblem for a search direction:
\Statex \pc\pc $\Delta x_k \leftarrow \argmin_d \nabla g(x_k)^Td + \frac{1}{2}d^TH_kd +h(x_k+d).$
\State Select $t_k$ with a backtracking line search.
\State Update: $x_{k+1} \leftarrow x_k + t_k\Delta x_k$.
\Until{stopping conditions are satisfied.}
\end{algorithmic}
\end{algorithm}

\subsection{Inexact proximal Newton-type methods}
\label{sec:inexact-proxnewton}

Inexact proximal Newton-type methods solve the subproblems
\eqref{eq:proxnewton-search-dir-1} approximately to obtain inexact
search directions. These methods can be more efficient than their
exact counterparts because they require less computation per
iteration. Indeed, many practical implementations of proximal
Newton-type methods such as \texttt{\texttt{glmnet}}, \texttt{newGLMNET}, and
QUIC solve the subproblems inexactly.
In practice, how inexactly we solve the subproblem is
critical to the efficiency and reliability of the method. The
practical implementations just mentioned
use a variety of heuristics to decide. % how accurately to solve the subproblem.
Although these methods perform admirably in practice,
there are few results on how the inexact subproblem solutions
affect their convergence behavior.

We now describe an adaptive stopping condition for the subproblem.  In
section \ref{sec:convergence-results} we analyze the convergence
behavior of inexact Newton-type methods, and in
section~\ref{sec:experiments} we conduct computational experiments to
compare the performance of our stopping condition against commonly
used heuristics.

Our adaptive stopping condition follows the one used by
\emph{inexact Newton-type methods} for minimizing smooth functions:
\begin{align}
  \norm{\nabla \hat{g}_k(x_k + \Delta x_k)} \le \eta_k\norm{\nabla g(x_k)},
  \label{eq:inexact-newton-stopping-condition}
\end{align}
where $\eta_k$ is a \emph{forcing term} that requires the left-hand side
to be small.  We generalize % \eqref{eq:inexact-newton-stopping-condition}
the condition to composite functions by substituting composite gradients into
\eqref{eq:inexact-newton-stopping-condition}: if $H_k \preceq MI$ for
some $M > 0$, we require
\begin{align}
  \|G_{\hat{f}_k/M}(x_k + \Delta x_k)\|
  \le \eta_k\norm{G_{f/M}(x_k)}.
  \label{eq:adaptive-stopping-condition}
\end{align}
We set $\eta_k$ based on how well $\hat{G}_{k-1}$ approximates $G$
near $x_k$: if $mI\preceq H_k$ for some $m > 0$, we require
\begin{align}
  \eta_k =\min\,\left\{\frac{m}{2},
    \frac{\|G_{\hat{f}_{k-1}/M}(x_k)-G_{f/M}(x_k)\|}{\norm{G_{f/M}(x_{k-1})}} \right\}.
  \label{eq:forcing-term}
\end{align}
This choice due to Eisenstat and Walker \cite{eisenstat1996choosing}
yields desirable convergence results and performs admirably in
practice.

Intuitively, we should solve the subproblem exactly if (i) $x_k$ is
close to the optimal solution, and (ii) $\hat{f}_k$ is a good model of
$f$ near $x_k$. If (i), we seek to preserve the fast local convergence
behavior of proximal Newton-type methods; if (ii), then minimizing
$\hat{f}_k$ is a good surrogate for minimizing $f$. In these cases,
\eqref{eq:adaptive-stopping-condition} and \eqref{eq:forcing-term}
are small, so the subproblem is solved accurately.

We can derive an expression like
\eqref{eq:proxnewton-search-direction-2} for an inexact search
direction in terms of an explicit gradient, an implicit subgradient,
and a residual term $r_k$. This reveals connections to the inexact
Newton search direction in the case of smooth problems. The adaptive
stopping condition \eqref{eq:adaptive-stopping-condition} is equivalent to
\begin{align*}
  0&\in G_{\hat{f}_k}(x_k + \Delta x_k) + r_k
\\ &= \nabla \hat{g}_k(x_k + \Delta x_k) +\partial h(x_k
      + \Delta x_k + G_{\hat{f}_k}(x_k + \Delta x_k)) + r_k
\\ &= \nabla g(x_k) + H_k\Delta x_k +\partial h(x_k
      + \Delta x_k + G_{\hat{f}_k}(x_k + \Delta x_k)) + r_k
\end{align*}
for some $r_k$ such that $\norm{r_k} \le \eta_k\norm{G_f(x_k)}$. Thus
an inexact search direction satisfies
\begin{align}
  H_k\Delta x_k \in -\nabla g(x_k) - \partial h(x_k
  + \Delta x_k + G_{\hat{f}_k}(x_k + \Delta x_k)) + r_k.
  \label{eq:inexact-proxnewton-search-direction}
\end{align}

Recently, Byrd et al.\ \cite{byrd2013inexact} analyze the inexact proximal Newton method with a more stringent adaptive stopping condition 
\BEQ
\|G_{\hat{f}_k/M}(x_k + \Delta x_k)\| \le \eta_k\norm{G_{f/M}(x_k)}\text{ and }\hat{f}_k(x_k+\Delta x_k) - \hat{f}_k(x_k) \le \beta\lambda_k
\label{eq:stringent-adaptive-stopping-condition}
\EEQ
for some $\beta\in(0,\frac12)$. The second condition is a sufficient descent condition on the subproblem. When $h$ is the $\ell_1$ norm, they show the inexact proximal Newton method with the stopping criterion \eqref{eq:stringent-adaptive-stopping-condition}
\BNUM
\item converges globally
\item eventually accepts the unit step length  
\item converges linearly or superlinearly depending on the forcing terms.
\ENUM 
Although the first two results generalize readily to composite functions with a generic $h$, the third result depends on the separability of the $\ell_1$
norm, and do not apply to generic composite functions. Since most practical implementations such as \cite{hsieh2011sparse} and \cite{yuan2012improved} more closely correspond to \eqref{eq:adaptive-stopping-condition}, we state our results for the adaptive stopping condition that does not impose sufficient descent. However our local convergence result combined with their first two results, imply the inexact proximal Newton method with stopping condition \eqref{eq:stringent-adaptive-stopping-condition} globally converges, and converges linearly or superlinearly (depending on the forcing term) for a generic $h$.

\section{Convergence of proximal Newton-type methods}
\label{sec:convergence-results}

We now analyze the convergence behavior of proximal Newton-type
methods.  In section \ref{sec:global-convergence} we show that
proximal Newton-type methods converge globally when the subproblems
are solved exactly.  In sections \ref{sec:convergence-proxnewton}
and \ref{sec:convergence-prox-quasinewton} we show that proximal
Newton-type methods and proximal quasi-Newton methods converge
$q$-quadratically and $q$-superlinearly subject to standard
assumptions on the smooth part $g$. In section
\ref{sec:convergence-inexact-proxnewton}, we show that the inexact
proximal Newton method converges
\begin{itemize}
\item $q$-linearly when the forcing terms $\eta_k$ are uniformly
  smaller than the inverse of the Lipschitz constant of $G_f$;
\item $q$-superlinearly when the forcing terms $\eta_k$ are chosen
  according to \eqref{eq:forcing-term}.
\end{itemize}
% If the nonsmooth part $h$ is zero, then we recover the result of
% Dembo et al. on the linear convergence of the inexact Newton method.

\subsection*{Notation} $G_f$ is the composite
gradient step on the composite function $f$, and $\lambda_k$ is the
decrease in $f$ predicted by linear extrapolation on $g$ at $x_k$
along the search direction $\Delta x_k$:
\[
  \lambda_k := \nabla g(x_k)^T\Delta x_k+h(x_k+\Delta x_k)-h(x_k).
\]
$L_1$, $L_2$, and $L_{G_f}$ are the Lipschitz constants
of $\nabla g$, $\nabla^2 g$, and $G_f$ respectively, while $m$
and $M$ are the (uniform) strong convexity and smoothness constants
for the $\hat{g}_k$'s, \ie, $mI\preceq H_k \preceq MI$. If we set
$H_k = \nabla^2 g(x_k)$, then $m$ and $M$ are also the strong
convexity and strong smoothness constants of $g$.

\subsection{Global convergence}
\label{sec:global-convergence}
Our first result shows proximal Newton-type methods converge globally
to some optimal solution $x^\star$. There are many similar results; 
\eg, those in \cite[section 4]{patriksson1999nonlinear},
and Theorem \ref{thm:global-convergence} is neither the first nor the
most general. We include this result because the proof is simple and
intuitive.
% The proof assumes the subproblems are solved \emph{exactly}. 
We assume 
\begin{enumerate}
\item $f$ is a closed, convex function and $\inf_x\{f(x)\mid x\in\dom f\}$ is attained;
\item the $H_k$'s are (uniformly) positive definite; \ie, $H_k \succeq
mI$ for some $m > 0$.
\end{enumerate}
The second assumption ensures that the methods are executable,
\ie, there exist step lengths that satisfy the sufficient descent
condition (cf.\ Lemma \ref{lem:acceptable-step-lengths}).

\begin{theorem}
  \label{thm:global-convergence}
  Suppose $f$ is a closed convex function, and $\inf_x\{f(x)\mid
  x\in\dom f\}$ is attained at some $x^\star$. If $H_k \succeq mI$ for
  some $m > 0$ and the subproblems \eqref{eq:proxnewton-search-dir-1}
  are solved exactly, then $x_k$ converges to an optimal solution
  starting at any $x_0\in\dom f$.
\end{theorem}

\begin{proof}
  The sequence $\{f(x_k)\}$ is decreasing because $\Delta x_k$ are descent
  directions \eqref{eq:descent} and
  there exist step lengths satisfying sufficient descent
  \eqref{eq:sufficient-descent} (cf.\ Lemma \ref{lem:acceptable-step-lengths}):
  \[
	f(x_k)-f(x_{k+1}) \le \alpha t_k\lambda_k \le 0.
  \]
  The sequence $\{f(x_k)\}$ must converge to some limit because $f$ is
  closed and the optimal value is attained. Thus $t_k\lambda_k$ must
  decay to zero. The step lengths $t_k$ are bounded away from zero
  because sufficiently small step lengths attain sufficient
  descent. Thus $\lambda_k$ must decay to zero. By
  \eqref{eq:search-direction-properties-2}, we deduce that $\Delta
  x_k$ also converges to zero:
  \[
	\norm{\Delta x_k}^2 \le \frac{1}{m} \Delta x_k^TH_k\Delta x_k
						\le -\frac{1}{m}\lambda_k.
  \]
  Since the search direction $\Delta x_k$ is zero if and only if $x$
  is an optimal solution (cf.\ Proposition
  \ref{prop:first-order-conditions}), $x_k$ must converge to some
  optimal solution $x^\star$.
\end{proof}

%In the next two subsections, we analyze the convergence rate of the
%proximal Newton and quasi-Newton methods. 

\subsection{Local convergence of the proximal Newton method}
\label{sec:convergence-proxnewton}

In this section and section \ref{sec:convergence-prox-quasinewton} we
study the convergence rate of the proximal Newton and proximal
quasi-Newton methods when the subproblems are solved exactly.  First,
we state our assumptions on the problem.

\begin{definition}
\label{def:strong-convex}
A function $f$ is strongly convex with constant $m$ if 
\begin{equation}
f(y) \ge f(x) + \nabla f(x)^T(y - x) + \frac{m}{2}\norm{x - y}^2\text{ for any }x,y.
\label{eq:strong-convex}
\end{equation}
\end{definition}

We assume the smooth part $g$ is strongly convex with constant $m$. This is a standard assumption in the analysis of Newton-type methods
for minimizing smooth functions. If $f$ is twice-continuously
differentiable, then this assumption is equivalent to $\nabla^2 f(x)
\succeq m I$ for any $x$. For our purposes, this assumption can
usually be relaxed by only requiring \eqref{eq:strong-convex} for any
$x$ and $y$ close to $x^\star$.

We also assume the gradient of the smooth part $\nabla g$ and Hessian $\nabla^2g$
are Lipschitz continuous with constants $L_1$ and $L_2$. The
assumption on $\nabla^2 g$ is standard in the analysis of Newton-type
methods for minimizing smooth functions. For our purposes, this
assumption can be relaxed by only requiring
\eqref{eq:lipschitz-continuity} for any $x$ and $y$ close to
$x^\star$.

The proximal Newton method uses the exact Hessian $H_k = \nabla^2
g(x_k)$ in the second-order model of $f$. This method converges
$q$-quadratically:
\[
  \norm{x_{k+1}-x^\star} = O\big(\norm{x_k-x^\star}^2\big),
\]
subject to standard assumptions on the smooth part: that $g$ is
twice-continuously differentiable and strongly convex with constant
$m$, and $\nabla g$ and $\nabla^2 g$ are Lipschitz continuous with
constants $L_1$ and $L_2$.
We first prove an auxiliary result.
% : that unit step lengths satisfy the
% sufficient descent condition after sufficiently many iterations.

\begin{lemma}
  \label{lem:newton-unit-step}
%  Suppose (i) $g$ is twice-continuously differentiable and
%  (ii) $\nabla^2 g$ is Lipschitz continuous with constant $L_2$.
  If $H_k = \nabla^2 g(x_k)$, the unit step
  length satisfies the sufficient decrease condition
  \eqref{eq:sufficient-descent} for $k$ sufficiently large.
\end{lemma}

\begin{proof}
  Since $\nabla^2 g $ is Lipschitz continuous,
  \begin{equation*}
	g(x+\Delta x)  \le g(x) + \nabla g(x) ^{T} \Delta x
							+ \frac{1}{2} \Delta x^T\nabla^2 g(x)\Delta x
							+ \frac{L_2}{6} \norm{\Delta x}^3.
  \end{equation*}
  We add $h(x+\Delta x)$ to both sides to obtain
  \begin{align*}
	f(x+\Delta x) &\le g(x) + \nabla g(x) ^{T} \Delta x
							+ \frac{1}{2} \Delta x^T\nabla^2 g(x)\Delta x
%  \\              &\pc
                     +\frac{L_2}{6}\norm{\Delta x}^3 + h(x+\Delta x).
  \end{align*}
  We then add and subtract $h$ from the right-hand side to obtain
  \begin{align*}
	 f(x+\Delta x) &\le g(x) + h(x)  + \nabla g(x) ^{T} \Delta x+ h(x+\Delta x) - h(x)
  \\  &\pc+ \frac{1}{2} \Delta x^T\nabla^2 g(x)\Delta x +\frac{L_2}{6}\norm{\Delta x}^3
  \\  &\le f(x) +\lambda + \frac{1}{2}\Delta x^T\nabla^2 g(x)\Delta x
		 + \frac{L_2}{6}\norm{\Delta x}^3.
  \end{align*}
  Since $g$ is strongly convex and $\Delta x$ satisfies
  \eqref{eq:search-direction-properties-2}, we have
  \[
  f(x+\Delta x) \le f(x) +\lambda -\frac{1}{2}\lambda +\frac{L_2}{6m}\norm{\Delta x}\lambda.
  \]
  We rearrange to obtain
  \begin{align*}
    f(x+\Delta x)- f(x) &\le \frac{1}{2}\lambda -
    \frac{L_2}{6m}\lambda\norm{\Delta x} \le \left(\frac{1}{2}
      -\frac{L_2}{6m}\norm{\Delta x}\right)\lambda.
  \end{align*}
  We can show that $\Delta x_k$ decays to zero via the argument
  we used to prove Theorem~\ref{thm:global-convergence}. Hence, if $k$
  is sufficiently large, $f(x_k+\Delta x_k) -f(x_k) < \frac{1}{2} \lambda_k$.
\end{proof}

%We use this result to prove the proximal Newton method converges
%$q$-quadratically to $x^\star$ subject to standard assumptions on $g$.

\begin{theorem}
  \label{thm:newton-quadratic-convergence}
  % Suppose (i) $g$ is twice-continuously differentiable and strongly
  % convex with constant $m$ and (ii) $\nabla g$ and $\nabla^2 g$ are
  % Lipschitz continuous with constants $L_1$ and $L_2$.
  % Then the
  The proximal Newton method converges $q$-quadratically to $x^\star$.
\end{theorem}

\begin{proof}
  Since the assumptions of Lemma \ref{lem:newton-unit-step} are
  satisfied, unit step lengths satisfy the sufficient descent
  condition:
  \[
    x_{k+1} = x_k + \Delta x_k =
    \prox_h^{\nabla^2 g(x_k)}\left(x_k-\nabla^2 g(x_k)^{-1}\nabla g(x_k)\right).
  \]
  Since scaled proximal mappings are firmly non-expansive in the
  scaled norm, we have
  \begin{align*}
  \norm{x_{k+1} - x^\star}_{\nabla^2 g(x_k)}
     &= \big\|\prox_h^{\nabla^2 g(x_k)}(x_k-\nabla^2 g(x_k)^{-1}\nabla g(x_k))
  \\ &\pc\pc - \prox_h^{\nabla^2 g(x_k)}(x^\star
      - \nabla^2 g(x_k)^{-1}\nabla g(x^\star))\big\|_{\nabla^2 g(x_k)}
  \\ &\le \norm{x_k - x^\star +
      \nabla^2 g(x_k)^{-1}(\nabla g(x^\star) - \nabla g(x_k))}_{\nabla^2 g(x_k)}
   % &= \norm{H_k^{1/2}x_k-H_k^{-1/2}\nabla g(x_k) - H_k^{1/2}x^\star+H_k^{-1/2}\nabla g(x^\star)}
  \\ &\le \frac{1}{\sqrt{m}}\norm{\nabla^2 g(x_k)(x_k-x^\star)-\nabla g(x_k)+\nabla g(x^\star)}.
  \end{align*}
  Since $\nabla^2 g$ is Lipschitz continuous, %with constant $L_2$,
  we have
  \[
    \norm{\nabla^2 g(x_k)(x_k-x^\star) - \nabla g(x_k)+\nabla g(x^\star)}
                                    \le \frac{L_2}{2}\norm{x_k-x^\star}^2
  \]
  and we deduce that $x_k$ converges to $x^\star$ quadratically:
  \[
  \norm{x_{k+1}-x^\star} \le
  \frac{1}{\sqrt{m}}\norm{x_{k+1}-x^\star}_{\nabla^2 g(x_k)} \le
  \frac{L_2}{2m} \norm{x_k-x^\star}^2.
  \]
\end{proof}

\subsection{Local convergence of proximal quasi-Newton methods}
\label{sec:convergence-prox-quasinewton}

If the sequence $\{H_k\}$ satisfies the Dennis-Mor\'{e} criterion
\cite{dennis1974characterization}, namely
\begin{align}
  \frac{\norm{\left(H_{k} -\nabla^2 g(x^\star)\right)(x_{k+1}-x_k)}}{\norm{x_{k+1}-x_k}}\to 0,
  \label{eq:dennis-more}
\end{align}
we can prove that a proximal quasi-Newton method converges
$q$-superlinearly:
\[
  \norm{x_{k+1}-x^\star} \le o(\norm{x_k-x^\star}).
\]
Again we assume that $g$ is twice-continuously differentiable and
strongly convex with constant $m$, and $\nabla g$ and $\nabla^2 g$ are
Lipschitz continuous with constants $L_1$ and $L_2$.  These are the
assumptions required to prove that quasi-Newton methods for minimizing
smooth functions converge superlinearly.

First, we prove two auxiliary results: that (i) step lengths of unity
satisfy the sufficient descent condition after sufficiently many
iterations, and (ii) the proximal quasi-Newton step is close to the
proximal Newton step.

\begin{lemma}
  \label{lem:quasinewton-unit-step}
%  Suppose (i) $g$ is twice-continuously differentiable and
%  (ii) $\nabla^2 g$ is Lipschitz continuous with constant $L_2$.
  If $\{H_k\}$ satisfy the Dennis-Mor\'{e}
  criterion and $mI \preceq H_k \preceq MI$ for some $0 < m \le M$, then the unit step length satisfies the sufficient
  descent condition \eqref{eq:sufficient-descent} after sufficiently
  many iterations.
\end{lemma}

\begin{proof}
  The proof is very similar to the proof of Lemma
  \ref{lem:newton-unit-step}, and we defer the details to Appendix
  \ref{sec:proofs}.
\end{proof}

The proof of the next result mimics the analysis of Tseng and Yun
\cite{tseng2009coordinate}.

\begin{proposition}
  \label{prop:tseng}
  Suppose $H_1$ and $H_2$ are positive definite matrices with
  bounded eigenvalues: $mI\preceq H_1 \preceq MI$ and $m_2I\preceq
  H_2 \preceq M_2I$. Let $\Delta x_1$ and $\Delta x_2$ be
  the search directions generated using $H_1$ and $H_2$
  respectively:
  \begin{align*}
	\Delta x_1 &= \prox_h^{H_1}\left(x-H_1^{-1}\nabla g(x)\right) - x, \\
	\Delta x_2 &= \prox_h^{H_2}\left(x-H_2^{-1}\nabla g(x)\right) - x.
  \end{align*}
  Then there is some $\bar{\theta} > 0$ such that these two search
  directions satisfy
  \[
    \norm{\Delta x_1 - \Delta x_2} \le
    \sqrt{\frac{1+\bar{\theta}}{m_1}}\big\|(H_2-H_1)\Delta x_1\big\|^{1/2}
    \norm{\Delta x_1}^{1/2}.
  \]
\end{proposition}

\begin{proof}
  By \eqref{eq:proxnewton-search-dir-1} and Fermat's rule, $\Delta x_1$
  and $\Delta x_2$ are also the solutions to
  \begin{align*}
     \Delta x_1 &= \argmin_d\, \nabla g(x)^Td + \Delta x_1^TH_1d + h(x+d),
  \\ \Delta x_2 &= \argmin_d\, \nabla g(x)^Td + \Delta x_2^TH_2d + h(x+d).
  \end{align*}
  Thus $\Delta x_1$ and $\Delta x_2$ satisfy
  \begin{align*}
     &\nabla g(x)^T\Delta x_1 + \Delta x_1^TH_1\Delta x_1 + h(x+\Delta x_1)
  \\ &\pc\le \nabla g(x)^T\Delta x_2 + \Delta x_2^TH_1\Delta x_2 + h(x+\Delta x_2)
  \end{align*}
  and
  \begin{align*}
     &\nabla g(x)^T\Delta x_2 + \Delta x_2^TH_2\Delta x_2 + h(x+\Delta x_2)
  \\ &\pc\le \nabla g(x)^T\Delta x_1 + \Delta x_1^TH_2\Delta x_1 + h(x+\Delta x_1).
  \end{align*}
  We sum these two inequalities and rearrange to obtain
  \[
    \Delta x_1^TH_1\Delta x_1 - \Delta x_1^T(H_1+H_2)\Delta x_2
    + \Delta x_2^TH_2\Delta x_2 \le 0.
  \]
  We then complete the square on the left side and rearrange to obtain
  \begin{align*}
     &\Delta x_1^TH_1\Delta x_1 - 2\Delta x_1^TH_1\Delta x_2 + \Delta x_2^TH_1\Delta x_2
  \\ &\pc\le \Delta x_1^T(H_2-H_1)\Delta x_2 + \Delta x_2^T(H_1- H_2)\Delta x_2.
  \end{align*}
  The left side is $\norm{\Delta x_1 - \Delta x_2}_{H_1}^2$ and the
  eigenvalues of $H_1$ are bounded. Thus
  \begin{align}
    \norm{\Delta x_1 - \Delta x_2}
     &\le\frac{1}{\sqrt{m_1}} 
           \left(\Delta x_1^T(H_2-H_1)\Delta x_1
               + \Delta x_2^T(H_1- H_2)\Delta x_2\right)^{1/2} \nonumber
  \\ &\le \frac{1}{\sqrt{m_1}}
          \big\|(H_2-H_1)\Delta x_2\big\|^{1/2}
          (\norm{\Delta x_1} + \norm{\Delta x_2})^{1/2}.
	\label{eq:prop:tseng-1}
  \end{align}
  We use a result due to Tseng and Yun (cf.\ Lemma 3 in
  \cite{tseng2009coordinate}) to bound the term $\left(\norm{\Delta
      x_1} + \norm{\Delta x_2}\right)$. Let $P$ denote
  $H_2^{-1/2}H_1H_2^{-1/2}$. Then $\norm{\Delta x_1}$ and
  $\norm{\Delta x_2}$ satisfy
  \[
    \norm{\Delta x_1} \le
      \left(
        \frac{M_2\left( 1 + \lambda_{\max}(P)
                          + \sqrt{1-2\lambda_{\min}(P)+\lambda_{\max}(P)^2}
                 \right)}{2m}
      \right) \norm{\Delta x_2}.
  \]
  We denote the constant in parentheses by $\bar{\theta}$ and conclude that
  \begin{align}
     \norm{\Delta x_1} + \norm{\Delta x_2} \le (1 + \bar{\theta})\norm{\Delta x_2}.
     \label{eq:prop:tseng-2}
  \end{align}
  We substitute \eqref{eq:prop:tseng-2} into \eqref{eq:prop:tseng-1} to obtain
  \[
     \norm{\Delta x_1 - \Delta x_2}^2
       \le \sqrt{\frac{1+\bar{\theta}}{m_1}}
           \big\|(H_2-H_1)\Delta x_2\big\|^{1/2}
           \norm{\Delta x_2}^{1/2}.
  \]
\end{proof}

We use these two results to show proximal quasi-Newton methods
converge superlinearly to $x^\star$ subject to standard assumptions on
$g$ and $H_k$.

\begin{theorem} 
  \label{thm:superlinear-convergence}
  % Suppose (i) $g$ is twice-continuously differentiable and strongly
  % convex and
  % (ii) $\nabla g$ and $\nabla^2 g$ are Lipschitz continuous with
  % constants $L_1$ and $L_2$.
  If $\{H_k\}$ satisfy the Dennis-Mor\'{e} criterion and $mI \preceq
  H_k \preceq MI$ for some $0 < m \le M$, then a proximal quasi-Newton
  method converges $q$-superlinearly to $x^\star$.
\end{theorem}

\begin{proof}
  Since the assumptions of Lemma \ref{lem:quasinewton-unit-step} are
  satisfied, unit step lengths satisfy the sufficient descent
  condition after sufficiently many iterations:
  \[
    x_{k+1} = x_k + \Delta x_k.
  \]
  Since the proximal Newton method converges $q$-quadratically
  (cf.\ Theorem \ref{thm:newton-quadratic-convergence}),
  \begin{align}
	\norm{x_{k+1} - x^\star} &\le \norm{x_k + \Delta x_k^{\rm nt} - x^\star}
							 + \norm{\Delta x_k - \Delta x_k^{\rm nt}} \nonumber \\
	&\le \frac{L_2}{m}\norm{x_k^{\rm nt}-x^\star}^2 + \norm{\Delta x_k
	  - \Delta x_k^{\rm nt}},
	\label{eq:superlinear-convergence-1}
  \end{align}
  where $\Delta x_k^{\rm nt}$ denotes the proximal-Newton search
  direction and $x^{\rm nt} = x_k + \Delta x_k^{\rm nt}$.  We use
  Proposition \ref{prop:tseng} to bound the second term:
  \begin{align}
	\norm{\Delta x_k - \Delta x_k^{\rm nt}} \le
	\sqrt{\frac{1+\bar{\theta}}{m}}\norm{(\nabla^2 g(x_k)-H_k)\Delta
	  x_k}^{1/2}\norm{\Delta x_k}^{1/2}.
	\label{eq:superlinear-convergence-2}
  \end{align}
  Since the Hessian $\nabla^2 g$ is Lipschitz continuous and $\Delta
  x_k$ satisfies the Dennis-Mor\'{e} criterion, we have
  \begin{align*}
    \norm{\left(\nabla^2 g(x_k)-H_k\right)\Delta x_k}
     &\le \norm{\left(\nabla^2 g(x_k)- \nabla^2 g(x^\star)\right)\Delta x_k}
        + \norm{\left(\nabla^2 g(x^\star)-H_k\right)\Delta x_k}
  \\ &\le L_2\norm{x_k-x^\star}\norm{\Delta x_k} + o(\norm{\Delta x_k}).
  \end{align*}
  We know $\norm{\Delta x_k}$ is within some constant $\bar{\theta}_k$
  of $\|\Delta x_k^{\rm nt}\|$ (cf.\ Lemma 3 in
  \cite{tseng2009coordinate}). We also know the proximal Newton method
  converges $q$-quadratically. Thus
  \begin{align*}
    \norm{\Delta x_k} &\le \bar{\theta}_k\norm{\Delta x_k^{\rm nt}}
                         = \bar{\theta}_k\norm{x_{k+1}^{\rm nt} - x_k}
  \\ &\le \bar{\theta}_k\left(\norm{x_{k+1}^{\rm nt} -x^\star} + \norm{x_k - x^\star}\right)
  \\ &\le O\big(\norm{x_k-x^\star}^2\big) + \bar{\theta}_k\norm{x_k-x^\star}.
  \end{align*}
  We substitute these expressions into
  \eqref{eq:superlinear-convergence-2} to obtain
  \begin{align*}
     \norm{\Delta x_k - \Delta x_k^{\rm nt}} = o(\norm{x_k-x^\star}).
  \end{align*}
  We substitute this expression into
  \eqref{eq:superlinear-convergence-1} to obtain
  \[
    \norm{x_{k+1} - x^\star}
       \le \frac{L_2}{m} \norm{x_k^{\rm nt}-x^\star}^2
           + o(\norm{x_k-x^\star}),
  \]
  and we deduce that $x_k$ converges to $x^\star$ superlinearly.
\end{proof}

\subsection{Local convergence of the inexact proximal Newton method}
\label{sec:convergence-inexact-proxnewton}
Because subproblem \eqref{eq:proxnewton-search-dir-1} is rarely
solved exactly, we now analyze the adaptive stopping criterion
\eqref{eq:adaptive-stopping-condition}:
\[
  \|G_{\hat{f}_k/M}(x_k + \Delta x_k)\|  \le \eta_k\norm{G_{f/M}(x_k)}.
\]
We show that the inexact proximal Newton method with unit step length (i)
converges $q$-linearly if the forcing terms $\eta_k$ are smaller than
some $\bar{\eta}$, and (ii) converges $q$-superlinearly if the forcing
terms decay to zero.
% When $h$ is zero, we recover the result of Dembo et al. on the
% linear convergence of the inexact Newton method
% \cite{dembo1982inexact} (cf.\ Remark \ref{rem:dembo}).

As before, we assume (i) $g$ is twice-continuously differentiable and
strongly convex with constant $m$, and (ii) $g$ and $\nabla^2 g$ are
Lipschitz continuous with constants $L_1$ and $L_2$. We also assume
(iii) $x_k$ is close to $x^\star$, and (iv) the unit step length is
eventually accepted. These are the assumptions made by Dembo et al.\
and Eisenstat and Walker \cite{dembo1982inexact,eisenstat1996choosing}
in their analysis of \emph{inexact Newton methods} for minimizing
smooth functions. 

First, we prove two auxiliary results that show (i) $G_{\hat{f}_k}$ is
a good approximation to $G_f$, and (ii) $G_{\hat{f}_k}$ inherits the
strong monotonicity of $\nabla\hat{g}$.

\begin{lemma}
\label{lem:G-newton-approx}
We have
  \(
    \|G_f(x) - G_{\hat{f}_k}(x)\| \le \frac{L_2}{2}\norm{x - x_k}^2.
  \)
\end{lemma}

\begin{proof}
The proximal mapping is non-expansive:
\begin{align*}
   \|G_f(x) - G_{\hat{f}_k}(x)\|
      &\le \norm{\prox_h(x - \nabla g(x)) - \prox_h(x - \nabla\hat{g}_k(x))}
%\\    &
\le \norm{\nabla g(x) - \nabla\hat{g}_k(x)}.
\end{align*}
Since $\nabla g(x)$ and $\nabla^2 g(x_k)$ are Lipschitz continuous,
\[
  \norm{\nabla g(x) - \nabla\hat{g}_k(x)}
     \le \norm{\nabla g(x) - \nabla g(x_k) - \nabla^2 g(x_k)(x - x_k)}
     \le \frac{L_2}{2}\norm{x - x_k}^2.
\]
Combining the two inequalities gives the desired result.
\end{proof}

The proof of the next result mimics the analysis of Byrd et al. \cite{byrd2013inexact}. 

\begin{lemma}
\label{lem:G-strongly-monotone}
% Suppose (i) $g$ is twice-continuously differentiable, (ii) $\nabla g$
% is Lipschitz continuous with constant $L_1$, and (iii) $g$ is strongly
% convex with constant $m$.
% Then
$G_{tf}(x)$ with $t \le \frac{1}{L_1}$
is strongly monotone with constant $\frac{m}{2}$, \ie,
\begin{equation}
  (x-y)^T(G_{tf}(x) - G_{tf}(y)) \ge
     \frac{m}{2}\norm{x - y}^2\text{ for }t\le\frac{1}{L_1}.
\label{eq:G-strongly-monotone-1}
\end{equation}
\end{lemma}

\begin{proof}
The composite gradient step on $f$ has the form
\begin{align*}
  G_{tf}(x) = \frac{1}{t}(x - \prox_{th}(x - t\nabla g(x)))
\end{align*}
(\cf{} \eqref{eq:composite-gradient-step}). We decompose
$\prox_{th}(x - t\nabla g(x))$ (by Moreau's decomposition) to obtain
\[
  G_{tf}(x) = \nabla g(x) + \frac{1}{t}\prox_{(th)^*}(x - t\nabla g(x)).
\]
Thus $G_{tf}(x) - G_{tf}(y)$ has the form
\begin{align}
   &G_{tf}(x) - G_{tf}(y)   \nonumber
\\ &\pc= \nabla g(x) - \nabla g(y) + \frac{1}{t}
       \left(\prox_{(th)^*}(x - t\nabla g(x)) - \prox_{(th)^*}(y - t\nabla g(y))\right).
\label{eq:G-strongly-monotone-2}
\end{align}
Let $w=\prox_{(th)^*}(x - t\nabla g(x)) - \prox_{(th)^*}(y - t\nabla
g(y))$ and
\[
  d = x - t\nabla g(x) - (y - t\nabla g(y))
    = (x - y) - t(\nabla g(x) - t\nabla g(y)).
\]
We express \eqref{eq:G-strongly-monotone-2} in terms of $W =
\frac{ww^T}{w^Td}$ to obtain
\[
  G_{tf}(x) - G_{tf}(y) = \nabla g(x) - \nabla g(y) + \frac{w}{t}
    = \nabla g(x) - \nabla g(y) + \frac1tWd.
\]
We multiply by $x-y$ to obtain
\begin{align}
   &(x-y)^T(G_{tf}(x) - G_{tf}(y))  \nonumber
\\ &\pc =(x-y)^T (\nabla g(x) - \nabla g(y)) + \frac1t(x-y)^T Wd \nonumber
\\ &\pc =(x-y)^T (\nabla g(x) - \nabla g(y)) + \frac1t(x-y)^T
         W(x - y - t(\nabla g(x) - \nabla g(y))) 
\label{eq:G-strongly-monotone-3}
\end{align}
%By the integral form of the mean value theorem, we have
%$$
%\nabla g(x) - \nabla g(y) =\int_{0}^1 \nabla^2g (x+\alpha (x-y)) (x-y)\ d\alpha.
%$$
Let $H(\alpha)= \nabla^2g (x+\alpha (x-y))$. By the mean value
theorem, we have
\begin{align}
   &(x-y)^T(G_{tf}(x) - G_{tf}(y))     \nonumber
\\ &\pc = \int_{0}^{1} (x-y)^T \left( H(\alpha) -WH(\alpha)
          + \frac{1}{t} W\right) (x-y)\ d\alpha   \nonumber
\\ &\pc = \int_{0}^{1} (x-y)^T \left( H(\alpha) -
             \frac{1}{2} (WH(\alpha) +H(\alpha) W) +\frac{1}{t} W \right) (x-y)\ d\alpha
\label{eq:G-strongly-monotone-4}
\end{align}
To show \eqref{eq:G-strongly-monotone-1}, we must show that $H(\alpha)
+ \frac1t W -\frac12(WH(\alpha) + H(\alpha)W)$ is positive definite
for $t \le\frac{1}{L_1}$.  We rearrange $(\sqrt{t} H(\alpha) -
\frac{1}{\sqrt{t}} W)(\sqrt{t} H(\alpha) - \frac{1}{\sqrt{t}}
W)\succeq 0$ to obtain
\[
  t H(\alpha)^2 +\frac{1}{t} W^2 \succeq WH(\alpha) +H(\alpha)W,
\]
and we substitute this expression into
\eqref{eq:G-strongly-monotone-4} to obtain
\begin{align*}
   &(x-y)^T(G_{tf}(x) - G_{tf}(y))
\\ &\pc\ge \int_{0}^{1} (x-y)^T \left(
       H(\alpha) - \frac{t}{2} H(\alpha)^2
                 + \frac{1}{t} \bigl(W - \frac{1}{2} W^2 \bigr)
                                \right)
       (x-y)\ d\alpha.
\end{align*}
Since $\prox_{(th)^*}$ is firmly non-expansive, we have $\norm{w}^2
\le d^Tw$ and
\[
  W = \frac{ww^T}{w^Td} = \frac{\norm{w}^2}{w^Td}\frac{ww^T}{\norm{w}^2}\preceq I.
\]
% We know $\prox_{(th)^*}$ is firmly non-expansive, hence $\norm{w}^2
% \le d^Tw$ and
%$$
%W = \frac{ww^T}{w^Td} = \frac{\norm{w}^2}{w^Td}\frac{ww^T}{\norm{w}^2}\preceq I.
%$$
% We also know $W$ is positive semidefinite, so $W - \frac12W^2$ is
% also positive semidefinite. If we set $t = \frac1L$, then the
% eigenvalues of $H - \frac{t}{2}H^2$ are
%$$
% \lambda_i - \frac{t}{2}\lambda_i^2 = \lambda_i -
% \frac{\lambda_i^2}{2L_1} \ge \frac{\lambda_i}{2} > \frac{m}{2},
%$$
% where $\lambda_i,i=1,\dots,n$ are the eigenvalues of $H$. This means
% $G_{f/L_1}(x)$ is strongly monotone with constant $\frac{m}{2}$.
Since $W$ is positive semidefinite and $W\preceq I$, $W - W^2$ is positive
semidefinite and
\[
  (x-y)^T(G_{tf}(x) - G_{tf}(y)) \ge \int_{0}^{1} (x-y)^T \left(
  H(\alpha) -\frac{t}{2} H(\alpha)^2 \right) (x-y)\ d\alpha.
\]
If we set $t \le \frac{1}{L_1}$, the eigenvalues of $H(\alpha) -
\frac{t}{2}H(\alpha)^2$ are
\[
  \lambda_i(\alpha) - \frac{t}{2}\lambda_i(\alpha)^2 \ge
  \lambda_i(\alpha) - \frac{\lambda_i(\alpha)^2}{2L_1} \ge
  \frac{\lambda_i(\alpha)}{2} > \frac{m}{2},
\]
where $\lambda_i(\alpha),i=1,\dots,n$ are the eigenvalues of $H(\alpha)$.
We deduce that
\[
  (x-y)^T(G_{tf}(x) - G_{tf}(y)) \ge \frac{m}{2} \norm{x-y}^2.
\]
\end{proof}

We use these two results to show that the inexact proximal Newton
method with unit step lengths converges locally linearly or superlinearly depending on the forcing terms.

\begin{theorem}
\label{thm:inexact-newton-linear-convergence}
  Suppose
% (i) $g$ is twice-continuously differentiable and strongly
% convex with constant $m$, (ii) $\nabla g$ and $\nabla^2 g$ are
% Lipschitz continuous with constants $L_1$ and $L_2$, and (iii)
  $x_0$ is sufficiently close to $x^\star$.
  \begin{enumerate}
  \item If $\eta_k$ is smaller than some $\bar{\eta} < \frac{m}{2}$, an
    inexact proximal Newton method with unit step lengths converges $q$-linearly to $x^\star$.

  \item If $\eta_k$ decays to zero, an inexact proximal Newton method with unit step lengths
    converges $q$-super\-linearly to $x^\star$.
  \end{enumerate}
\end{theorem}

\begin{proof}
  The local model $\hat{f}_k$ is strongly convex with constant $m$.
  According to Lemma \ref{lem:G-strongly-monotone},
  $G_{\hat{f}_k/L_1}$ is strongly monotone with constant
  $\frac{m}{2}$:
  \[
  (x-y)^T\left(G_{\hat{f}_k/L_1}(x) - G_{\hat{f}_k/L_1}(y)\right) \ge \frac{m}{2}\norm{x-y}^2.
  \]
  By the Cauchy-Schwarz inequality, we have
  \[
    \|G_{\hat{f}_k/L_1}(x) - G_{\hat{f}_k/L_1}(y)\| \ge \frac{m}{2}\norm{x-y}.
  \]
  We apply this result to $x_k +\Delta x_k$ and $x^\star$ to obtain
  \begin{equation}
    \norm{x_{k+1} - x^\star} =\norm{x_k + \Delta x_k - x^\star}
       \le \frac{2}{m}\|G_{\hat{f}_k/L_1}(x_k
           + \Delta x_k) - G_{\hat{f}_k/L_1}(x^\star)\|.
    \label{eq:inexact-newton-linear-convergence-0}
  \end{equation}
  Let $r_k$ be the residual $- G_{\hat{f}_k./L_1}(x_k + \Delta x_k)$.
  The adaptive stopping condition
  \eqref{eq:adaptive-stopping-condition} requires $\norm{r_k} \le
  \eta_k\|G_{f/L_1}(x_k)\|$. We substitute this expression into
  \eqref{eq:inexact-newton-linear-convergence-0} to obtain
  \begin{align}
     \norm{x_{k+1} - x^\star}
        &\le \frac{2}{m}\|-G_{\hat{f}_k/L_1}(x^\star) -r_k\|         \nonumber
  \\    &\le \frac{2}{m}(\|G_{\hat{f}_k/L_1}(x^\star)\|+\norm{r_k})  \nonumber
  \\    &\le \frac{2}{m} (\|G_{\hat{f}_k/L_1}(x^\star)\|+\eta_k\norm{G_{f/L_1}(x_k)}.
  \label{eq:inexact-newton-linear-convergence-1}
  \end{align}
Applying Lemma \ref{lem:G-newton-approx} to $f/L_1$ and $\hat f_k
  /L_1$ gives
\[
\|G_{\hat{f}_k/L_1}(x^\star)\|
  \le \frac12\frac{L_2}{L_1}\norm{x_k - x^\star}^2+\norm{
    G_{f/L_1}(x^\star)}=\frac12\frac{L_2}{L_1}\norm{x_k -
    x^\star}^2.
\]
 We substitute this bound into
  \eqref{eq:inexact-newton-linear-convergence-1} to obtain
  \begin{align*}
    \norm{x_{k+1} - x^\star} &\le \frac{2}{m} \left(
                                \frac{L_2}{2L_1}\norm{x_k - x^\star}^2
                                + \eta_k\norm{G_{f/L_1}(x_k)}
                                             \right)
  \\                         &\le \frac{L_2}{mL_1}
                                \norm{x_k - x^\star}^2
                              + \frac{2\eta_k}{m}\norm{x_k - x^\star}.
\end{align*}
We deduce that (i) $x_k$ converges $q$-linearly to $x^\star$ if $x_0$ is
sufficiently close to $x^\star$ and $\eta_k \le \bar{\eta}$ for some
$\bar{\eta} < \frac{m}{2}$, and (ii) $x_k$ converges $q$-superlinearly
to $x^\star$ if $x_0$ is sufficiently close to $x^\star$ and $\eta_k$
decays to zero.
\end{proof}

Finally, we justify our choice of forcing terms: if we choose $\eta_k$
according to \eqref{eq:forcing-term}, then the inexact proximal Newton
method converges $q$-superlinearly. When minimizing smooth functions, we recover the result of Eisenstat and Walker on choosing forcing terms in an inexact Newton method \cite{eisenstat1996choosing}.
%% theorem rewritten 12-20-2013, by jason
\begin{theorem}
  \label{thm:forcing-term-1}
  Suppose
  % (i) $g$ is twice-continuously differentiable and strongly convex with constant $m$,
  % (ii) $\nabla g$ and $\nabla^2 g$ are Lipschitz continuous with
  % constants $L_1$ and $L_2$, and (iii)
  $x_0$ is sufficiently close to $x^\star$. If we choose $\eta_k$
  according to \eqref{eq:forcing-term}, then the inexact proximal
  Newton method with unit step lengths converges $q$-superlinearly.
\end{theorem}
\begin{proof}
To show superlinear convergence, we must show
\begin{align}
\frac{\|G_{\hat{f}_{k-1}/L_1}(x_k)-G_{f/L_1}(x_k)\|}{\norm{G_{f/L_1}(x_{k-1})}}\to 0
\label{eq:eta-0}.
\end{align}
By Lemma \ref{lem:G-lipschitz}, we have 
\begin{align*}
   \|G_{\hat{f}_{k-1}/L_1}(x_k)-G_{f/L_1}(x_k)\|
      &\le \frac12\frac{L_2}{L_1} \norm{x_k-x_{k-1}}^2
\\    &\le \frac12\frac{L_2}{L_1}
           \left(\norm{x_k - x^\star} +\norm{x^\star - x_{k-1}}\right)^2.
\end{align*}
By Lemma \ref{lem:G-strongly-monotone}, we also have 
\[
  \norm{G_{f/L_1}(x_{k-1})} = \norm{G_{f/L_1} (x_{k-1}) - G_{f/L_1}
    (x^\star)} \ge \frac{m}{2} \norm{x_{k-1} - x^\star}.
\]
We substitute these expressions into \eqref{eq:eta-0} to obtain
\begin{align*}
    &\frac{\|G_{\hat{f}_{k-1}/L_1}(x_k)-G_{f/L_1}(x_k)\|}
          {\norm{G_{f/L_1}(x_{k-1})}}
\\  &\pc\le \frac{\frac12\frac{L_2}{L_1}
                  \left(\norm{x_k - x^\star} + \norm{x^\star - x_{k-1}}\right)^2}
                 {\frac{m}{2} \norm{x_{k-1} - x^\star}}
\\  &\pc=   \frac{1}{m}\frac{L_2}{L_1} 
            \frac{\norm{x_k - x^\star} + \norm{x_{k-1}-x^\star}}
                 {\norm{x_{k-1} -x^\star}}
                 \left( \norm{x_k - x^\star} +\norm{x_{k-1}-x^\star}\right)
\\ &\pc=    \frac{1}{m}\frac{L_2}{L_1}
            \left(1 + \frac{\norm{x_k - x^\star}}{\norm{x_{k-1} -x ^\star}} \right)
            \left(\norm{x_k - x^\star} + \norm{x_{k-1}-x^\star}\right).
\end{align*}
By Theorem \ref{thm:inexact-newton-linear-convergence}, we have
$\frac{\norm{x_k - x^\star}}{\norm{x_{k-1} -x ^\star}} < 1$ and
\[
  \frac{\|G_{\hat{f}_{k-1}/M}(x_k)-G_{f/M}(x_k)\|}
       {\norm{G_{f/M}(x_{k-1})}}
  \le \frac{2}{m} \frac{L_2}{M}
      \left( \norm{x_k - x^\star} + \norm{x_{k-1}-x^\star}\right).
\]
We deduce (with Theorem \ref{thm:inexact-newton-linear-convergence})
that the inexact proximal Newton method with adaptive stopping
condition \eqref{eq:adaptive-stopping-condition} converges
$q$-superlinearly.
\end{proof}

\section{Computational experiments}
\label{sec:experiments}

First we explore how inexact search directions affect the convergence
behavior of proximal Newton-type methods on a problem in
bioinfomatics. We show that choosing the forcing terms according to
\eqref{eq:forcing-term} avoids ``oversolving'' the subproblem.  Then
we demonstrate the performance of proximal Newton-type methods using a
problem in statistical learning. We show that the methods are suited
to problems with expensive smooth function evaluations.

\subsection{Inverse covariance estimation}

Suppose i.i.d.\ samples $x^{(1)},\dots,x^{(m)}$ are from
a Gaussian Markov random field (MRF) with mean zero and unknown inverse covariance
matrix $\bar{\Theta}$:
\[
  \Prob(x;\bar{\Theta}) \propto \exp(x^T\bar{\Theta} x/2 - \logdet(\bar{\Theta})).
\]
We seek a sparse maximum likelihood estimate of the inverse covariance matrix:
\begin{align}
  \hat{\Theta} := \argmin_{\Theta \in \reals^{n\times
	  n}}\,\trace\left(\hat{\Sigma}\Theta\right) - \log\det(\Theta) +
	\lambda\norm{\mathrm{vec}(\Theta)}_1,
  \label{eq:l1-logdet}
\end{align}
where $\hat{\Sigma}$ denotes the sample covariance matrix. We
regularize using an entry-wise $\ell_1$ norm to avoid overfitting the
data and to promote sparse estimates. The parameter $\lambda$ balances
goodness-of-fit and sparsity.

We use two datasets: (i) Estrogen, a gene expression dataset
consisting of 682 probe sets collected from 158 patients, and (ii)
Leukemia, another gene expression dataset consisting of 1255 genes
from 72 patients.\footnote{These datasets are available from
  \url{http://www.math.nus.edu.sg/~mattohkc/} with the SPINCOVSE
  package.} The features of Estrogen were converted to log-scale and
normalized to have zero mean and unit variance. The regularization
parameter $\lambda$ was chosen to match the values used in
\cite{rolfs2012iterative}.

We solve the inverse covariance estimation problem
\eqref{eq:l1-logdet} using a proximal BFGS method, \ie, $H_k$ is
updated according to the BFGS updating formula.
(The proximal Newton method would be
computationally very expensive on these large datasets.)
%The BFGS update satisfies the Dennis-Mor\'{e} criterion \cite{dennis1977quasi}, so we expect the proximal BFGS method to converge superlinearly. 
To explore how inexact search directions affect the convergence
behavior, we use three rules to decide how accurately to solve
subproblem \eqref{eq:proxnewton-search-dir-1}:
\begin{enumerate}
\item adaptive: stop when the adaptive stopping condition
  \eqref{eq:adaptive-stopping-condition} is satisfied;
\item exact: solve the subproblem accurately (``exactly'');
\item stop after 10 iterations.
\end{enumerate} 
We use the TFOCS implementation of FISTA to solve the subproblem. We
plot relative suboptimality versus function evaluations and time on
the Estrogen dataset in Figure \ref{fig:estrogen} and on the Leukemia
dataset in Figure \ref{fig:leukemia}.

\begin{figure}
   \begin{subfigure}{0.5\textwidth}
	\includegraphics[width=\textwidth]{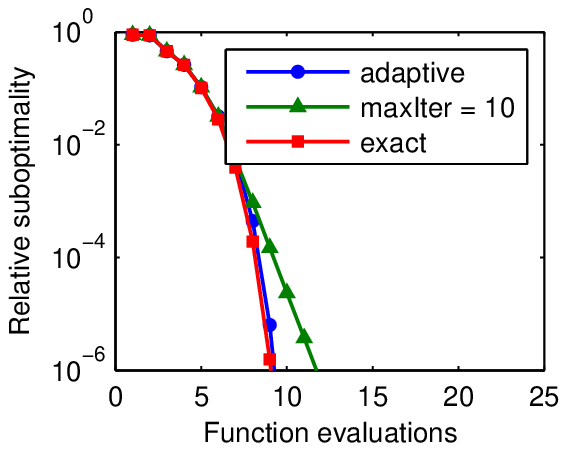}
   \end{subfigure}%
	~
   \begin{subfigure}{0.5\textwidth}
	 \includegraphics[width=\textwidth]{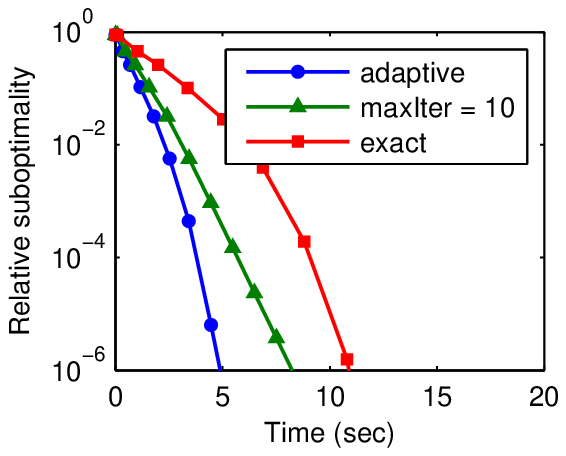}
   \end{subfigure}
   \caption{Inverse covariance estimation problem (Estrogen
     dataset). Convergence behavior of proximal BFGS method with three
     subproblem stopping conditions.}
   \label{fig:estrogen}
\end{figure}

\begin{figure}
   \begin{subfigure}{0.5\textwidth}
	\includegraphics[width=\textwidth]{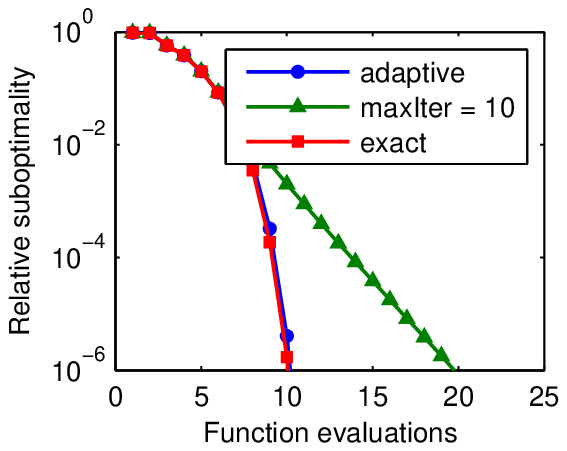}
   \end{subfigure}%
	~
   \begin{subfigure}{0.5\textwidth}
	 \includegraphics[width=\textwidth]{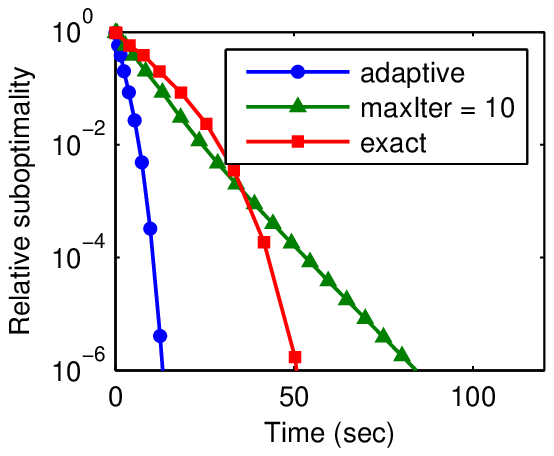}
   \end{subfigure}
   \caption{Inverse covariance estimation problem (Leukemia
     dataset). Convergence behavior of proximal BFGS method with three
     subproblem stopping conditions.}
   \label{fig:leukemia}
\end{figure}
 
Although the conditions for superlinear convergence (cf.\ Theorem
\ref{thm:superlinear-convergence}) are not met ($\log\det$ is not
strongly convex), we empirically observe in Figures \ref{fig:estrogen}
and \ref{fig:leukemia} that a proximal BFGS method transitions from
linear to superlinear convergence. This transition is characteristic
of BFGS and other quasi-Newton methods with superlinear convergence.

On both datasets, the exact stopping condition yields the fastest
convergence (ignoring computational expense per step), followed
closely by the adaptive stopping condition (see Figure
\ref{fig:estrogen} and \ref{fig:leukemia}). If we account for time per
step, then the adaptive stopping condition yields the fastest
convergence.  Note that the adaptive stopping condition yields
superlinear convergence (like the exact proximal BFGS method). The
third condition (stop after 10 iterations) yields only linear
convergence (like a first-order method), and its convergence rate is
affected by the condition number of $\hat{\Theta}$. On the Leukemia
dataset, the condition number is worse and the convergence is slower.

%\footnote{This behavior is typical of first-order methods \cite{rolfs2012iterative}} 

\subsection{Logistic regression}

Suppose we are given samples $x^{(1)},\dots,x^{(m)}$ with labels
$y^{(1)},\dots,y^{(m)}\in\{-1,1\}$. We fit a logit model to our data:
\begin{align}
  \minimize_{w \in \reals^n}\,\frac{1}{m}
    \sum_{i=1}^m \log(1+\exp(-y_i w^Tx_i)) + \lambda\norm{w}_1.
  \label{eq:l1-logistic}
\end{align}
Again, the regularization term $\norm{w}_1$ promotes sparse solutions
and $\lambda$ balances sparsity with goodness-of-fit.

We use two datasets: (i) \texttt{gisette}, a handwritten digits
dataset from the NIPS 2003 feature selection challenge ($n=5000$), and
(ii) \texttt{rcv1}, an archive of categorized news stories from
Reuters ($n=47,000$).\footnote{These datasets are available at
  \url{http://www.csie.ntu.edu.tw/~cjlin/libsvmtools/datasets}.} The
features of \texttt{gisette} have been scaled to be within the
interval $[-1,1]$, and those of \texttt{rcv1} have been scaled to be
unit vectors.  $\lambda$ matched the value reported in
\cite{yuan2012improved}, where it was chosen by five-fold cross
validation on the training set.

We compare a proximal L-BFGS method with SpaRSA and the TFOCS
implementation of FISTA (also Nesterov's 1983 method) on problem
\eqref{eq:l1-logistic}.
%We use these settings:
%\begin{itemize}
%\item FISTA: default settings;
%\item SpaRSA: 10 step non-monotone line search \cite{grippo1986nonmonotone} ($\alpha = 0.0001,\beta = 0.5$);
%\item Proximal L-BFGS method: 50 step limited memory update with backtracking line search ($\alpha = 0.0001,\beta = 0.5$) and inexact search directions (cf.\ \eqref{eq:adaptive-stopping-condition}); 
%\end{itemize}
We plot relative suboptimality versus function evaluations and time on
the \texttt{gisette} dataset in Figure \ref{fig:gisette} and on the
\texttt{rcv1} dataset in Figure \ref{fig:rcv1}.

\begin{figure}
   \begin{subfigure}{0.5\textwidth}
	\includegraphics[width=\textwidth]{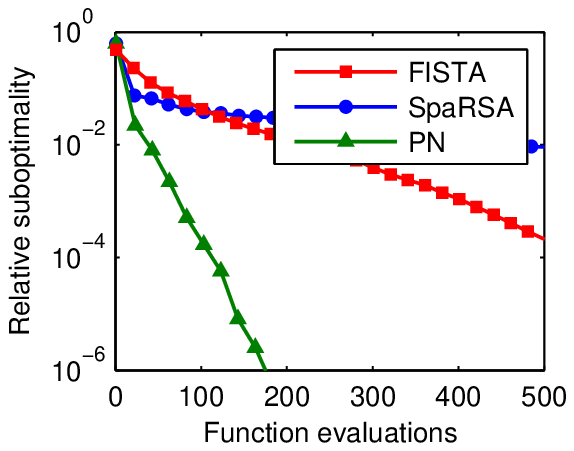}
   \end{subfigure}%
	~
   \begin{subfigure}{0.5\textwidth}
	 \includegraphics[width=\textwidth]{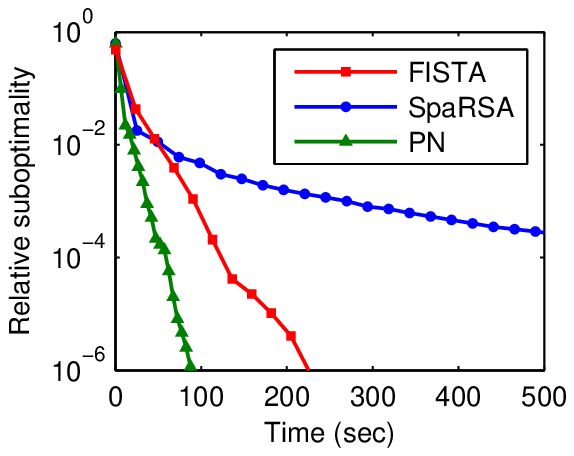}
   \end{subfigure}
   \caption{Logistic regression problem (\texttt{gisette}
     dataset). Proximal L-BFGS method (L = 50) versus FISTA and
     SpaRSA.}
   \label{fig:gisette}
\end{figure}

\begin{figure}
   \begin{subfigure}{0.5\textwidth}
	\includegraphics[width=\textwidth]{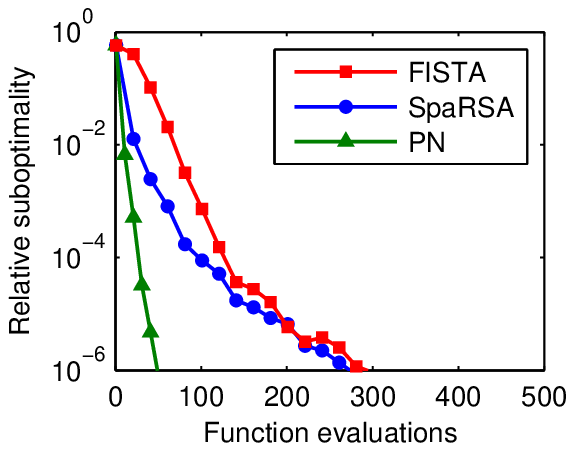}
   \end{subfigure}%
	~
   \begin{subfigure}{0.5\textwidth}
	 \includegraphics[width=\textwidth]{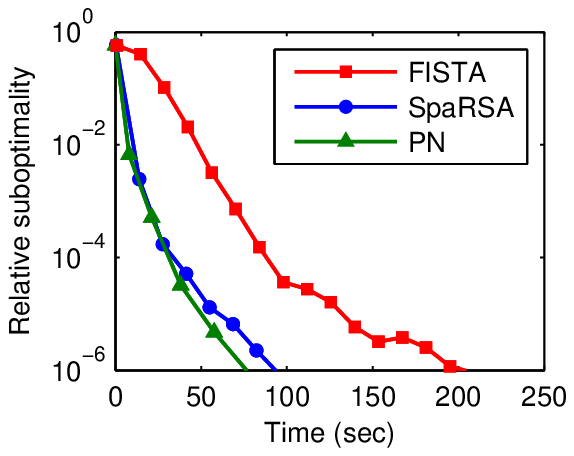}
   \end{subfigure}
   \caption{Logistic regression problem (\texttt{rcv1}
     dataset). Proximal L-BFGS method (L = 50) versus FISTA and
     SpaRSA.}
   \label{fig:rcv1}
\end{figure}

The smooth part of the function requires many expensive exp/log
evaluations. On the dense \texttt{gisette} dataset (30
million nonzero entries in a $6000 \times 5000$ design matrix), evaluating
$g$ dominates the computational cost. The proximal L-BFGS method
clearly outperforms the other methods because the computational
expense is shifted to solving the subproblems, whose objective
functions are cheap to evaluate (see Figure \ref{fig:gisette}). On the
sparse \texttt{rcv1} dataset (40 million nonzero entries in a $542000
\times 47000$ design matrix), the evaluation of $g$ makes up a
smaller portion of the total cost, and the proximal L-BFGS method
barely outperforms SpaRSA (see Figure \ref{fig:rcv1}).

\subsection{Software: PNOPT}

The methods described have been incorporated into a \textsc{Matlab}
package PNOPT (Proximal Newton OPTimizer, pronounced pee-en-opt) and
are publicly available from the Systems Optimization Laboratory (SOL)%
\footnote{\url{http://www.stanford.edu/group/SOL/}}. PNOPT shares an
interface with the software package TFOCS \cite{becker2011templates}
and is compatible with the function generators included with TFOCS. We
refer to the SOL website for details about PNOPT.

\section{Conclusion}

Given the popularity of first-order methods for minimizing composite
functions, there has been a flurry of activity around the development of
Newton-type methods for minimizing composite functions
\cite{hsieh2011sparse, becker2012quasi, olsen2012newton}. We analyze
proximal Newton-type methods for such functions and show that they have
several benefits over first-order methods: 
\begin{enumerate}
\item They converge rapidly near the optimal
  solution, and can produce a solution of high accuracy.
\item They scale well with problem size.
\item The proximal Newton method is insensitive to the choice of
  coordinate system and to the condition number of the level sets of
  the objective.
\end{enumerate}

Proximal Newton-type methods can readily handle composite functions
where $g$ is not convex, although care must be taken to ensure
$\hat{g}_k$ remains strongly convex. The convergence analysis could
be modified to give global convergence (to stationary points) and
convergence rates near stationary points. We defer these extensions to future work.

The main disadvantage of proximal Newton-type methods is the cost of
solving the subproblems.  We have shown that it is possible to reduce
the cost and retain the fast convergence rate by solving the
subproblems inexactly.  We hope our results will kindle further
interest in proximal Newton-type methods as an alternative to
first-order methods and interior point methods for minimizing
composite functions.

\section*{Acknowledgements}

We thank Santiago Akle, Trevor Hastie, Nick Henderson, Qihang Lin,
Xiao Lin, Qiang Liu, Ernest Ryu, Ed Schmerling, Mark Schmidt, Carlos
Sing-Long, Walter Murray, and four anonymous referees for their
insightful comments and suggestions.
J. Lee was supported by a National Defense Science and Engineering
Graduate Fellowship and a
Stanford Graduate Fellowship.
Y. Sun and M. Saunders were partially supported by the Department
of Energy through the Scientific Discovery through Advanced
Computing program under award DE-FG02-09ER25917,
and by the National Institute of General Medical Sciences of
the National Institutes of Health under award U01GM102098.
M.~Saunders was also partially by
the Office of Naval Research under award N00014-11-1-0067.
The content is solely the responsibility of the authors and does not
necessarily represent the official views of the funding agencies.

\appendix

\section{Proof of Lemma \ref{lem:quasinewton-unit-step}}
\label{sec:proofs}

\begin{lemma}
  Suppose (i) $g$ is twice-continuously differentiable and strongly
  convex with constant $m$, and (ii) $\nabla^2 g$ is Lipschitz
  continuous with constant $L_2$.  If the sequence $\{H_k\}$ satisfies
  the Dennis-Mor\'{e} criterion and $mI \preceq H_k \preceq MI$ for
  some $0 < m \le M$, then the unit step length satisfies the
  sufficient descent condition \eqref{eq:sufficient-descent} after
  sufficiently many iterations.
\end{lemma}

\begin{proof}
  Since $\nabla^2 g $ is Lipschitz,
  \begin{equation*}
    g(x+\Delta x)  \le g(x) + \nabla g(x) ^{T} \Delta x
    + \frac{1}{2}\Delta x^T\nabla^2 g(x)\Delta x
    + \frac{L_2}{6}\norm{\Delta x}^3.
  \end{equation*}
  We add $h(x+\Delta x)$ to both sides to obtain
  \begin{align*}
    f(x+\Delta x) &\le g(x) + \nabla g(x) ^{T} \Delta x
                     + \frac{1}{2} \Delta x^T\nabla^2 g(x)\Delta x
  \\              &\pc+\frac{L_2}{6}\norm{\Delta x}^3 + h(x+\Delta x).
  \end{align*}
  We then add and subtract $h(x)$ from the right-hand side to obtain
  \begin{align*}
    f(x+\Delta x) &\le g(x) + h(x)  + \nabla g(x) ^{T} \Delta x+ h(x+\Delta x) - h(x)
  \\	&\pc+ \frac{1}{2} \Delta x^T\nabla^2 g(x)\Delta x +\frac{L_2}{6}\norm{\Delta x}^3
  \\	&\le f(x) +\lambda + \frac{1}{2}\Delta x^T\nabla^2 g(x)\Delta x
                           + \frac{L_2}{6}\norm{\Delta x}^3
  \\	&\le f(x) +\lambda + \frac{1}{2}\Delta x^T\nabla^2 g(x)\Delta x
                           + \frac{L_2}{6m}\norm{\Delta x}\lambda,
  \end{align*}
  where we use \eqref{eq:search-direction-properties-2}. We add and
  subtract $\frac{1}{2}\Delta x^TH\Delta x$ to yield
  \begin{align}
    f(x+\Delta x)
     &\le f(x) + \lambda
               + \frac{1}{2}\Delta x^T\left(\nabla^2 g(x) - H\right)\Delta x
               + \frac{1}{2}\Delta x^TH\Delta x
               + \frac{L_2}{6m}\norm{\Delta x}\lambda \nonumber
  \\ &\le f(x) + \lambda
               + \frac{1}{2}\Delta x^T\left(\nabla^2 g(x) - H\right)\Delta x 
               -\frac{1}{2}\lambda+\frac{L_2}{6m}\norm{\Delta x}\lambda,
	\label{eq:quasi-newton-unit-step-1}
  \end{align}
  where we again use \eqref{eq:search-direction-properties-2}.
  Since $\nabla^2 g$ is Lipschitz continuous and the search
  direction $\Delta x$ satisfies the Dennis-Mor\'{e} criterion,
  \begin{align*}
	&\frac{1}{2}\Delta x^T \left(\nabla^2 g(x)-H \right)\Delta x \\
	&\pc=   \frac{1}{2}\Delta x^T \left(\nabla^2 g(x)-\nabla^2 g(x^\star)\right)\Delta x
		  + \frac{1}{2}\Delta x^T \left(\nabla^2 g(x^\star)-H \right)\Delta x \\
	&\pc\le\frac{1}{2}\norm{\nabla^2 g(x)-\nabla^2 g(x^\star)} \norm{\Delta x}^2
		  + \frac{1}{2}\norm{\left(\nabla^2 g(x^\star)-H \right)\Delta x}\norm{\Delta x} \\
	&\pc\le \frac{L_2}{2}\norm{x-x^\star}\norm{\Delta x}^2 +
	o\big(\norm{\Delta x}^2\big).
  \end{align*}
  We substitute this expression into \eqref{eq:quasi-newton-unit-step-1}
  and rearrange to obtain
  \[
    f(x+\Delta x) \le f(x) + \frac{1}{2}\lambda
	+ o\big(\norm{\Delta x}^2\big) + \frac{L_2}{6m}\norm{\Delta x}\lambda.
  \]
  We can show $\Delta x_k$ converges to zero via the argument used in
  the proof of Theorem \ref{thm:global-convergence}. Hence, for $k$
  sufficiently large, $ f(x_k+\Delta x_k) -f(x_k) \le \frac{1}{2} \lambda_k$.
\end{proof}
%
%\begin{remark*}{\ref{rem:dembo}}
%When minimizing smooth functions ($h$ is zero), $G_{\hat{f}_k/L_1} = \frac{1}{L_1}\nabla\hat{g}_k$ is Lipschitz continuous with constant $\frac{L_1}{L_1} = 1$. In place of \eqref{eq:inexact-newton-linear-convergence-1}, we have 
%\begin{align*}
%\norm{x_{k+1} - x^\star} &\le \frac{1}{m}\left(\frac{L_2}{L_1}\norm{x_k - x^\star}^2 + \eta_k\norm{x_k - x^\star}\right) \\
%& = \frac{1}{m}\left(\frac{L_2}{L_1}\norm{x_k - x^\star} + \eta_k\right)\norm{x_k - x^\star}.
%\end{align*}
%If $\eta_k \le \bar{\eta}$ for some $\bar{\eta} < 1$, then $\frac{1}{m}\left(\frac{L_2}{L_1}\norm{x_k - x^\star} + \eta_k\right) < 1$ when $x_k$ is close to $x^\star$ and $x_k$ converges $q$-linearly to $x^\star$. This is the classical result of Dembo et al. \cite{dembo1982inexact}.
%\end{remark*}

\frenchspacing
\bibliographystyle{siam}
\bibliography{proxnewton}
\end{document}